\documentclass[11pt]{article}
\usepackage{amsmath, amssymb, graphicx, url, algorithm2e}
\usepackage[round]{natbib}
\usepackage{amsthm}
\usepackage{times}
\usepackage{mathtools}
\usepackage{bbm}
\usepackage{thmtools,thm-restate}
\usepackage{paralist}
\usepackage{wrapfig}
\usepackage[left=1in,right=1in,top=1in,bottom=1in]{geometry}

\newtheorem{lemma}{Lemma}

\newtheorem{theorem}{Theorem}
\newtheorem{definition}{Definition}

\DeclareMathOperator{\prg}{prog}
\newcommand{\prog}[1]{\prg_{#1}}
\newcommand{\mathth}[0]{^\textrm{th}}

\makeatletter
\def\set@curr@file#1{\def\@curr@file{#1}} %temp workaround for 2019 latex release
\makeatother 
\usepackage[load-configurations=version-1]{siunitx}

\newcommand{\R}{\mathbb{R}}
\newcommand{\E}{\mathbb{E}}

\DeclarePairedDelimiter{\abs}{\lvert}{\rvert} %
\DeclarePairedDelimiter{\brk}{[}{]}
\DeclarePairedDelimiter{\crl}{\{}{\}}
\DeclarePairedDelimiter{\prn}{(}{)}
\DeclarePairedDelimiter{\nrm}{\|}{\|}
\DeclarePairedDelimiter{\norm}{\|}{\|}

\let\P\undefined

\DeclareMathOperator{\P}{\mathbb{P}}

\DeclareMathOperator*{\argmin}{arg\,min}

\newcommand{\mc}[1]{\mathcal{#1}}

\def\ddefloop#1{\ifx\ddefloop#1\else\ddef{#1}\expandafter\ddefloop\fi}
\def\ddef#1{\expandafter\def\csname 
bb#1\endcsname{\ensuremath{\mathbb{#1}}}}
\ddefloop ABCDEFGHIJKLMNOPQRSTUVWXYZ\ddefloop
\def\ddefloop#1{\ifx\ddefloop#1\else\ddef{#1}\expandafter\ddefloop\fi}
\def\ddef#1{\expandafter\def\csname 
b#1\endcsname{\ensuremath{\mathbf{#1}}}}
\ddefloop ABCDEFGHIJKLMNOPQRSTUVWXYZ\ddefloop
\def\ddef#1{\expandafter\def\csname 
c#1\endcsname{\ensuremath{\mathcal{#1}}}}
\ddefloop ABCDEFGHIJKLMNOPQRSTUVWXYZ\ddefloop
\def\ddef#1{\expandafter\def\csname 
h#1\endcsname{\ensuremath{\widehat{#1}}}}
\ddefloop ABCDEFGHIJKLMNOPQRSTUVWXYZ\ddefloop
\def\ddef#1{\expandafter\def\csname 
hc#1\endcsname{\ensuremath{\widehat{\mathcal{#1}}}}}
\ddefloop ABCDEFGHIJKLMNOPQRSTUVWXYZ\ddefloop
\def\ddef#1{\expandafter\def\csname 
t#1\endcsname{\ensuremath{\widetilde{#1}}}}
\ddefloop ABCDEFGHIJKLMNOPQRSTUVWXYZ\ddefloop
\def\ddef#1{\expandafter\def\csname 
tc#1\endcsname{\ensuremath{\widetilde{\mathcal{#1}}}}}
\ddefloop ABCDEFGHIJKLMNOPQRSTUVWXYZ\ddefloop

\newcommand{\indicator}[1]{\mathbbm{1}{\crl{#1}}}    %Indicator
\usepackage{nicefrac}

\newcommand{\inner}[2]{\left\langle #1,\, #2 \right\rangle}
\newcommand{\removed}[1]{}

\newenvironment{proofsketch}{\noindent\textbf{Proof Sketch}}{}%{\hfill$\blacksquare$}

\usepackage[colorlinks,allcolors=blue]{hyperref} 

\title{The Min-Max Complexity of Distributed Stochastic\\ Convex Optimization with Intermittent Communication}
\title{\vspace{-2em}\rule{\linewidth}{2pt}\\ \textbf{The Min-Max Complexity of Distributed Stochastic\\ Convex Optimization with Intermittent Communication} \\\rule[8pt]{\linewidth}{1pt}\vspace{-0.7em}}
\author{\normalsize
\begin{minipage}{0.23\textwidth}
\centering
\textbf{Blake Woodworth}\\ 
Toyota Technological\\
Institute at Chicago \\
\small\url{blake@ttic.edu} \\ $ $
\end{minipage}
\begin{minipage}{0.23\textwidth}
\centering
\textbf{Brian Bullins} \\
Toyota Technological \\
Institute at Chicago \\
\small\url{bbullins@ttic.edu} \\ $ $
\end{minipage}
\begin{minipage}{0.23\textwidth}
\centering
\textbf{Ohad Shamir} \\
Weizmann Institute \\
of Science \\
\small\url{ohad.shamir@weizmann.ac.il} 
\end{minipage}
\begin{minipage}{0.23\textwidth}
\centering
\textbf{Nathan Srebro} \\
Toyota Technological \\
Institute at Chicago \\
\small\url{nati@ttic.edu} \\ $ $
\end{minipage}
\vspace{-4mm}
}
\date{}

\begin{document}

\maketitle

\begin{abstract}%
We resolve the min-max complexity of distributed stochastic convex optimization (up to a log factor) in the intermittent communication setting, where $M$ machines work in parallel over the course of $R$ rounds of communication to optimize the objective, and during each round of communication, each machine may sequentially compute $K$ stochastic gradient estimates. We present a novel lower bound with a matching upper bound that establishes an optimal algorithm. 
% We also describe several additional assumptions that can allow for circumventing the lower bound and developing faster algorithms.
\end{abstract}

\section{Introduction}\label{sec:intro}

% \os{as a selling point (especially for COLT), one can emphasize that understanding the optimal performance for stochastic convex optimization in this distributed setting has been an open problem for at least a decade (can cite ``Better Mini-Batch Algorithms via Accelerated Gradient Methods'' from 2011, and ``Distributed Stochastic Optimization and Learning'' from 2014 for example, even though the question/setting there were phrased a bit differently)\natiinline{I agree.  In fact, especially for a COLT paper, I would open with that}}

The min-max oracle complexity of stochastic convex optimization in a sequential (non-parallel) setting is very well-understood, and we have provably optimal algorithms that achieve the min-max complexity \citep{lan2012optimal,ghadimi2013optimal}.  However, we do not yet have an understanding of the min-max complexity of stochastic optimization in a {\em distributed} setting, where oracle queries and computation are performed by different workers, with limited communication between them. Perhaps the simplest, most basic, and most important distributed setting is that of {\em intermittent communication}.  

% \removed{With the increasing scale of modern optimization problems, for instance those that arise when training large machine learning models on very large datasets, it is important to design efficient distributed algorithms that can leverage parallelism for faster convergence. We study distributed stochastic first-order optimization in the prototypical ``intermittent communication setting'' \cite{woodworth2018graph} {\color{red}(other cites?)}, and we identify optimal algorithms and the optimal convergence rates in this setting.}

In the (homogeneous) intermittent communication setting, $M$ parallel workers are used to optimize a single objective over the course of $R$ rounds. During each round, each machine sequentially and locally computes $K$ independent unbiased stochastic gradients of the global objective, and then all the machines communicate with each other. This captures the natural setting where multiple parallel ``workers'' or ``machines''  are available, and computation on each worker is much faster than communication between workers. It includes applications ranging from optimization using multiple cores or GPUs, to using a cluster of servers, to Federated Learning\footnote{In a realistic Federated Learning setting, stochastic gradient estimates on the same machine might be correlated, or we might prefer thinking of a heterogeneous setting where each device has a different local objective. Nevertheless, much of the methodological and theoretical development in Federated Learning has been focused on the homogeneous intermittent communication setting we study here \citep[see][and citations therein]{kairouz2019advances}.} where workers are edge devices.

The intermittent communication setting has been widely studied for over a decade, with many optimization algorithms proposed and analyzed \citep{zinkevich2010parallelized,cotter2011better,dekel2012optimal,zhang2013divide,zhang2013communication,shamir2014distributed}, and obtaining new methods and improved analysis is still a very active area of research \citep{wang2017memory,stich2018local,wang2018cooperative,khaled2019better,haddadpour2019local,woodworth2020local}.
% \natinote{You should go back at least to the Dekel et al mini-batch SGD paper and John's average-at-the-end-SGD.  I would not emphasize local SGD or FedAvg in particular---they don't play an important role in this paper.  But I would cite all the recent papers to emphasize its an important question} 
However, despite these efforts, we do not yet know which methods are optimal, what the min-max complexity is, and what methodological or analytical improvements might allow us to make further progress.

Considerable effort has been made to formalize the setting and establish lower bounds for distributed optimization \citep{zhang2013information,arjevani2015communication,braverman2016communication} and here, we follow the graph-oracle formalization of \citet{woodworth2018graph}. However, a key issue in the existing literature is that known lower bounds for the intermittent communication setting depend only on the product $KR$ (i.e.~the total number of gradients computed on each machine over the course of optimization), and not on the number of rounds, $R$, and the number of gradients per round, $K$, separately. 

% Therefore, these lower bounds cannot distinguish between methods that communicate very frequently ($R = T$, $K = 1$) and methods that communicate just once ($R = 1$, $K = T$), and so cannot be used to obtain lower bounds on communication when the overall computational work is fixed.
% \natinote{Add discussion on ruling out optimiality with $R=1$, noting that this has actually been suggested by John, and is true for quadratic objectives, and based on existing literature cannot be ruled out for arbitrary smooth.  Emphasize this is the most basic question here, and even this we can't answer using existing lower bounds}
Thus, existing results cannot rule out the possibility that the optimal rate for fixed $T=KR$ can be achieved using only a single round of communication ($R=1$), since they do not distinguish between methods that communicate very frequently ($R = T$, $K = 1$) and methods that communicate just once ($R = 1$, $K = T$). The possibility that the optimal rate is achievable with $R=1$ was suggested by \citet{zhang2013communication}, and indeed \citet{woodworth2020local} proved that an algorithm that communicates just once is optimal in the special case of quadratic objectives. While it seems unlikely that a single round of communication suffices in the general case, none of our existing lower bounds are able to answer this extremely basic question.

In this paper, we resolve (up to a logarithmic factor) the minimax complexity of smooth, convex stochastic optimization in the (homogeneous) intermittent communication setting. Our main result in Section \ref{sec:minimax-rates} is a lower bound on the optimal rate of convergence and a matching upper bound. Interestingly, we show that the combination of two extremely simple and na\"ive methods based on an accelerated stochastic gradient descent (SGD) variant called AC-SA \citep{lan2012optimal} is optimal up to a logarithmic factor. Specifically, we show that the better of the following methods is optimal: ``Minibatch Accelerated SGD'' which executes $R$ steps of AC-SA using minibatch gradients of size $MK$, and ``Single-Machine Accelerated SGD'' which executes $KR$ steps of AC-SA on just one of the machines, completely ignoring the other $M-1$. 

These methods might seem to be horribly inefficient: Minibatch Accelerated SGD only performs one update per round of communication, and Single-Machine Accelerated SGD only uses one of the available workers!  This perceived inefficiency has prompted many attempts at developing improved methods which take multiple steps on each machine locally in parallel including, in particular, numerous analyses of Local SGD \citep{zinkevich2010parallelized,dekel2012optimal,stich2018local,haddadpour2019local,khaled2019better,woodworth2020local}.  Nevertheless, we establish that one or the other is optimal in every regime, so more sophisticated methods cannot yield improved guarantees for arbitrary smooth objectives. Our results therefore highlight an apparent dichotomy between exploiting the available parallelism but not the local computation (Minibatch Accelerated SGD) and exploiting the local computation but not the parallelism (Single-Machine Accelerated SGD).

% \os{should probably also discuss relation to last year's local-SGD paper, where we argued that local SGD improves over minibatch SGD in certain parameter regimes. So need to explain why there is no contradiction with the results here. \natiinline{I think this is better explained later on, but I agree it should be mentioned somewhere}}

Our lower bound applies quite broadly, including to the settings considered by much of the
% \natinote{Exceptions: DSVRG, Tengyu, your work on quadratics.  Is "nearly all" a fair quantifier here?} % Blake: fair enough, maybe not 
existing work on stochastic first-order optimization in the intermittent communication setting.  But, like many lower bounds, we should not interpret this to mean we cannot make progress. Rather, it indicates that we need to expand our model or modify our assumptions in order to develop better methods.
% \removed{but by know means does it preclude the possibility of  improved distributed optimization algorithms. Rather, the lower bound helps to identify additional structure in the problem that could allow for better methods.}
In Section \ref{sec:better-methods} we explore several additional assumptions that allow for circumventing our lower bound. These include when the third derivative of the objective is bounded (as in recent work by \citet{yuan2020federated}), when the objective has a certain statistical learning-like structure, or when the algorithm has access to a more powerful oracle.

% Intermittent Communication setting. Prototypical distributed optimization setting, interesting, widely studied, and there are many methods and analyses. However, the optimal complexity of optimization in the intermittent communication setting is unknown. Existing lower bounds make no distinction between $K$ and $R$, i.e.~we do not know to what extent you can substitute local computation for communication. 

% In this paper, we resolve the minimax complexity of optimization in the intermittent communication setting by proving a lower bound that matches existing upper bounds. In doing so, we prove that the combination of two extremely na\"ive baselines---Minibatch SGD and ``Single-Machine'' SGD---are optimal. 

% The setting of our lower bound applies to (almost?) all of the existing algorithms and analyses in the literature. Our result shows that algorithmic improvements in the intermittent communication setting will require leveraging additional assumptions about the objective, or using more powerful methods.

\section{Setting and Notation}\label{sec:setting-and-notation}
We aim to understand the fundamental limits of stochastic first-order algorithms in the intermittent communication setting. Accordingly, we consider a standard smooth, convex problem
\begin{equation}\label{eq:stochastic-objective}
\min_x F(x)
\end{equation}
where $F$ is convex, $\nrm{x^*} \leq B$, and $F$ is $H$-smooth, so for all $x,y$
\begin{equation}
F(x) + \inner{\nabla F(x)}{y - x} \leq F(y) \leq F(x) + \inner{\nabla F(x)}{y - x} + \frac{H}{2}\nrm{y-x}^2
\end{equation}
We consider algorithms that gain information about the objective via a stochastic gradient oracle $g$ with bounded variance\footnote{This assumption can be strong, and does not hold for natural problems like least squares regression \citep{nguyen2019new}, nevertheless, this strengthens rather than weakens our lower bound.}, which satisfies for all $x$
\begin{equation}
\E_z g(x;z) = \nabla F(x)\quad\textrm{and}\quad \E_z\nrm*{g(x;z) - \nabla F(x)}^2 \leq \sigma^2
\end{equation}
This is a well-studied class of optimization objectives: smooth, bounded, convex objectives with a bounded-variance stochastic gradient oracle. 

To understand optimal methods for this class of problems requires specifying a class of optimization algorithms. We consider intermittent communication algorithms, which attempt to optimize $F$ using $M$ parallel workers, each of which is allowed $K$ queries to $g$ in each of $R$ rounds of communication. Such intermittent communication algorithms can be formalized using the graph oracle framework of \citet{woodworth2018graph} which focuses on the dependence structure between different stochastic gradient computations. 

Finally, we are considering a ``homogeneous'' setting, where each of the machines have access to stochastic gradients from the same distribution, in contrast to the more challenging ``heterogeneous'' setting, where they come from \emph{different} distributions, which could arise in a machine learning context when each machine uses data from a different source. The heterogeneous setting is interesting, important, and widely studied, but we focus here on the more basic question of min-max rates for homogeneous distributed optimization. We point out that our lower bounds also apply to heterogeneous objectives since homogeneous optimization is a special case of heterogeneous optimization, and there are also some lower bounds specific to the heterogeneous setting \citep[e.g.][]{arjevani2015communication} but they do not apply to our setting.

% {\color{red} how does this fit in?}
% In contrast to some other work, we consider the ``homogeneous'' setting where each parallel worker has access to a stochastic gradient oracle $g$ with the same distribution conditional on $x$. This is in contrast to the ``heterogeneous'' setting where each machine might have stochastic gradients from a different distribution.
% \natinote{Does this paragraph go here or elsewhere?  Also: include citations, and explain that undertanding minmax complexity of homogenous is a more basic problem, and also that any lower bound applies also to any type of hetrogenous variant.}

\section{The Lower Bound}\label{sec:minimax-rates}

We now present our main result, which is a lower bound on what suboptimality can be guaranteed by any (possibly randomized) intermittent communication algorithm in the worst case:
\begin{restatable}{theorem}{mainlowerbound}\label{thm:lower-bound}
For any $H,B,\sigma,K,R > 0$ and $M\geq 2$, and any intermittent communication algorithm, there exists a convex, $H$-smooth objective which has a minimizer with norm at most $B$ in any dimension
\[
d \geq 2KR + \prn*{10^9\prn*{1 + KR + \prn*{\frac{HB}{\sigma}}^{3/2}M(KR)^{5/4}} + \frac{6144H^2B^2MKR}{\sigma^2}}\log(64MK^2R^2)
\]
and a stochastic gradient oracle, $g$, with $\E_z\nrm*{g(x;z) - \nabla F(x)}^2 \leq \sigma^2$ such that the algorithm's output will have error at least
\[
\E F(\hat{x}) - F^* \geq c\cdot\prn*{\frac{HB^2}{K^2R^2} + \min\crl*{\frac{\sigma B}{\sqrt{MKR}},\, HB^2} + \min\crl*{\frac{HB^2}{R^2 \log^2 M},\, \frac{\sigma B}{\sqrt{KR}}}}
\]
for a numerical constant $c$.
\end{restatable}
% \begin{restatable}{theorem}{lowerbound}\label{thm:lower-bound}
% For any $H,B,\sigma,K,R > 0$ and any $M \geq 2$ such that $H \geq \frac{\sigma}{B\sqrt{2MKR}}$,\footnote{We note that this restriction on $H$ is essentially without loss of generality since, by smoothness, $F(0) - F^* \leq \frac{1}{2}HB^2$ and it is well-known that any algorithm that uses at most $T$ stochastic gradients will suffer suboptimality at least $\frac{\sigma B}{\sqrt{T}}$ in the worst case \citep{nemirovskyyudin1983}.} there exists a convex, $H$-smooth objective $F$ with $\nrm{x^*}\leq B$ and a stochastic gradient oracle $g$ with $\E\nrm{g(x) - \nabla F(x)}^2 \leq \sigma^2$ for all $x$ such that with probability at least $\frac{1}{2}$, all of the oracle queries, $\crl{x^m_{k,r}}$, made by any distributed zero-respecting intermittent communication algorithm have suboptimality
% \[
% \min_{m,k,r}F(x^m_{k,r}) - F^* 
% \geq  c\cdot\brk*{\frac{HB^2}{K^2R^2} + \frac{\sigma B}{\sqrt{MKR}} + \min\crl*{\frac{HB^2}{R^2 \log^2 M},\, \frac{\sigma B}{\sqrt{KR}}}}
% \]
% \end{restatable}
\begin{proofsketch}
The first two terms of this lower bound follow directly from previous work \citep{woodworth2018graph}; the $\frac{HB^2}{K^2R^2}$ term corresponds to optimizing a function with a deterministic gradient oracle, and the $\frac{\sigma B}{\sqrt{MKR}}$ term is a very well-known statistical limit \citep[see, e.g.,][]{nemirovskyyudin1983}. The distinguishing feature of our lower bound is the second $\min$ term, which depends differently on $K$ than on $R$. For quadratics, the min-max complexity actually does depend only on the product $KR$, and is given by just the first two terms \citep{woodworth2020local}. Consequently, proving our lower bound necessitates going beyond quadratics (in contrast, all the lower bounds for sequential smooth convex optimization that we are aware of can be obtained using quadratics).
% \begin{figure}[tb]
% \centering
% {\includegraphics[height=4cm]{psi.pdf}}
% \caption{A plot of $\psi(x)$, which is nearly quadratic around zero and nearly linear away from zero.}
% \label{fig:psi-pic}
% \end{figure}
We therefore prove the Theorem using the following non-quadratic hard instance
\begin{equation}\label{eq:def-F-main}
F(x) = \psi'(-\zeta)x_1 + \psi(x_N) + \sum_{i=1}^{N-1}\psi(x_{i+1} - x_i)
\end{equation}
where $\psi:\R\to\R$ is defined as
\begin{equation}
\psi(x) := \frac{\sqrt{H}x}{2\beta}\arctan\prn*{\frac{\sqrt{H}\beta x}{2}} - \frac{1}{2\beta^2}\log\prn*{1+\frac{H\beta^2x^2}{4}}
\end{equation}
\begin{wrapfigure}{r}{0.23\textwidth}
\begin{center}
\vspace{-9mm}
{\includegraphics[width=\linewidth]{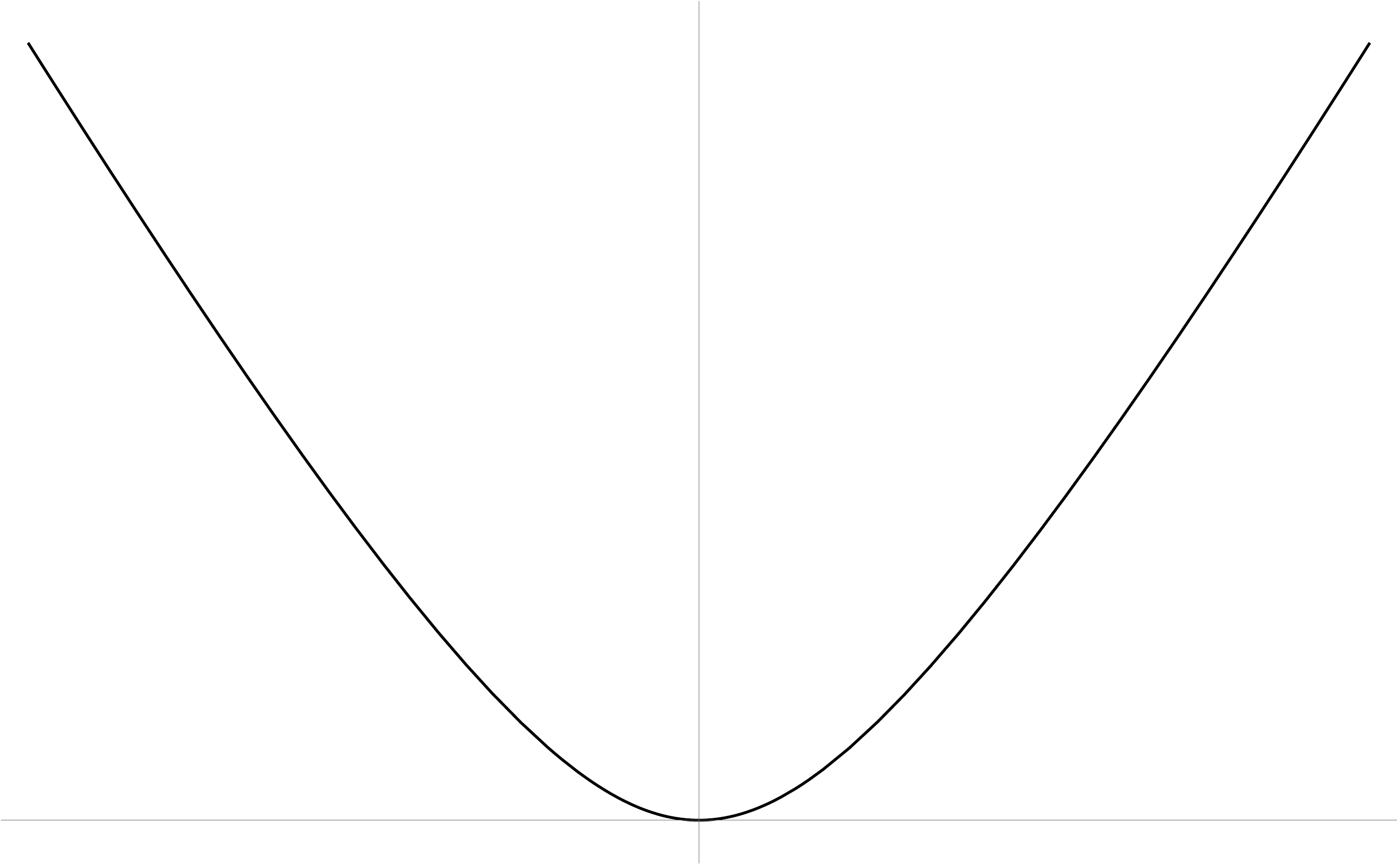}}
\small The function $\psi(x)$
\end{center}
\vspace{-7mm}
\end{wrapfigure}
and where $\beta$, $\zeta$, and $N$ are hyperparameters that are chosen depending on $H,B,\sigma,M,K,R$ so that $F$ satisfies the necessary conditions. This construction closely resembles the classic lower bound for deterministic first-order optimization of \citet{nesterov2004introductory}, which corresponds to $\psi(x) = x^2$. 
To describe our stochastic gradient oracle, we will use $\prog{\alpha}(x) := \max\crl*{j\,:\, \abs{x_j} > \alpha}$, which denotes the highest index of a coordinate of $x$ that is significantly non-zero. We also define $F^-$ to be equal to the objective with the $\prog{\alpha}(x)^{\textrm{th}}$ term removed:
\begin{equation}
F^-(x) = \psi'(-\zeta) x_1 + \psi(x_N) + \sum_{i=1}^{\prog{\alpha}(x)-1} \psi(x_{i+1} - x_i) + \sum_{i=\prog{\alpha}(x)+1}^{N-1} \psi(x_{i+1} - x_i)
\end{equation}
The stochastic gradient oracle for $F$ that we use then resembles
\begin{equation}\label{eq:def-g-main}
g(x) = \begin{cases}
\nabla F^-(x) & \textrm{with probability } 1-p \\
\nabla F(x) + \frac{1-p}{p}\prn*{\nabla F(x) - \nabla F^-(x)} & \textrm{with probability } p 
\end{cases}
\end{equation}
This stochastic gradient oracle is similar to the one used by \citet{arjevani2019lower} to prove lower bounds for non-convex optimization, and its key property is that $\P\brk*{\prog{\alpha}(g(x)) \leq \prog{\alpha}(x)} = 1 - p$. Therefore, each oracle access only reveals information about the next coordinate of the gradient the algorithm with probability $p$, and therefore the algorithm is essentially only able to make progress with probability $p$. 
The rest of the proof revolves around bounding the total progress of the algorithm and showing that if $\prog{\alpha}(x) \leq \frac{N}{2}$, then $x$ has high suboptimality.

Since each machine makes $KR$ sequential queries and only makes progress with probability $p$, the total progress scales like $KR \cdot p$. By taking $p$ smaller, we decrease the amount of progress made by the algorithm, and therefore increase the lower bound. Indeed, when $p \approx 1/K$, the algorithm only increases its progress by about $\log M$ per round, which gives rise to the key $(HB^2)/(R^2\log^2 M)$ term in the lower bound. However, we are constrained in how small we can take $p$ since our stochastic gradient oracle has variance 
\begin{equation}
\sup_x \E\nrm*{g(x) - \nabla F(x)}^2 \approx \frac{1}{p}\sup_x \psi'(x)^2
\end{equation}
This is where our choice of $\psi$ comes in. Specifically, we chose the function $\psi$ to be convex and smooth so that $F$ is, but we also made it Lipschitz:
\begin{equation}
\psi'(x) = \frac{\sqrt{H}}{2\beta}\arctan\prn*{\frac{\sqrt{H}\beta x}{2}} 
\in \brk*{-\frac{\pi\sqrt{H}}{4\beta},\,\frac{\pi\sqrt{H}}{4\beta}} 
\end{equation}
Notably, this Lipschitz bound on $\psi$, which implies a bound on $\nrm{\nabla F(x)}_{\infty}$, is the key non-quadratic property that allows for our lower bound. Since $\psi'$ is bounded, we are able to able to choose $p \approx H\sigma^{-2}\beta^{-2}$ without violating the variance constraint on the stochastic gradient oracle. Carefully balancing $\beta$ completes the argument. 

Another important aspect of our lower bound is that it applies to arbitrary randomized algorithms, rather than more restricted families of algorithms like ``zero-respecting'' methods (see Appendix \ref{app:thm5-proofs}). We therefore prove our theorem using techniques similar to \citet{woodworth16tight}, \citet{carmon2017lower1}, \citet{arjevani2019lower}, and others, who introduce a random rotation matrix, $U$; construct a hard instance like $F(U^\top x)$; and argue that any algorithm behaves almost as if it were zero-respecting. For further discussion of this proof technique, we refer readers to \citep{woodworth16tight,carmon2017lower1}. All of the details of the proof can be found in Appendices \ref{app:framework}-\ref{app:lower-bound-proof}.
\end{proofsketch}
% \os{should probably discuss how the $R,k$ components get to be decoupled in the proof, unlike previous lower bounds -- in fact, it's still a bit mysterious to me (I understand the math works out that way, but what's the intuition?)}

Theorem \ref{thm:lower-bound} also implies a lower bound for strongly convex objectives:
\begin{restatable}{corollary}{sclowerbound}\label{cor:strongly-convex-lower-bound}
There is a numerical constant, $c$, such that no intermittent communication algorithm can guarantee for any $H$-smooth, $\lambda$-strongly convex objective $F$ and stochastic gradient oracle with variance less than $\sigma^2$ that its output will have suboptimality 
\begin{multline*}
\E F(\hat{x}) - F^* \leq c\cdot\bigg(\frac{F(0) - F^*}{K^2R^2}\exp\prn*{-\sqrt{\frac{\lambda}{H}}KR} + \frac{\sigma^2}{\lambda MKR} \\
+ \min\crl*{\frac{F(0) - F^*}{R^2\log^2 M}\exp\prn*{-\sqrt{\frac{\lambda}{H}}R\log M},\, \frac{\sigma^2}{\lambda KR}}\bigg)
\end{multline*}
% where $c$ is a universal constant.
\end{restatable}
This lower bound is more limited than Theorem \ref{thm:lower-bound}, since we prove it using a reduction from convex to strongly convex optimization, rather than directly. We also do not expect the exponential terms to be tight. Nevertheless, the Corollary gives some indication of the optimal rate in the strongly convex setting and, as with Theorem \ref{thm:lower-bound}, it distinguishes between $R$ and $K$ unlike previous results. A simple proof can be found in Appendix \ref{app:lower-bound-proof}.

% \natinote{I would give much more details here.  This is a one-Theorem paper, you might as well give more details about it.  Spell out the stochastic oracle explicitly---this is also important for the discussion in Section 4.  Explain why progress is slow.  Also try to explain how you used the non-quadraticness.  Say where each of the four terms in the lower bound comes from.  Alternatively, or in addition, you can move more material from the appendix into the main text (first 12 pages)---enough for a relative "expert" reader, like Yair, Tengyu, John, Aaron Sidford, George Lan, etc, to see where everything falls, and enough for someone less super-specific, like Steve Wright, Shai B.D,, Manfred, Brendan, etc, to see where the specific lower bound is coming from.  Also, yes, you should include the lower bound statement also for the strongly convex case.  Again, it's a one-Theorem paper, at least make it comprehensive.}

\section{A Matching Upper Bound and an Optimal Algorithm}\label{sec:optimal-alg}

The lower bound in Theorem \ref{thm:lower-bound} is matched (up to $\log$ factors) by the combination of two simple distributed zero-respecting algorithms, which are distributed variants of an accelerated SGD algorithm called AC-SA due to \citet{lan2012optimal}. In the sequential setting, AC-SA algorithm maintains two iterates $y_t$ and $x_t$ which it updates according to
\begin{equation}\label{eq:AC-SA-updates}
\begin{aligned}
y_{t+1} &= y_t - \gamma_t g_t\prn*{\beta_t^{-1} y_t + (1 - \beta_t^{-1})x_t} \\
x_{t+1} &= \beta_t^{-1}y_{t+1} + (1 - \beta_t^{-1})x_t
\end{aligned}
\end{equation}
where $\gamma_t$ and $\beta_t$ are carefully chosen stepsize parameters. In the smooth, convex setting, this algorithm converges at a rate \citep[see Corollary 1,][]{lan2012optimal}
\begin{equation}\label{eq:AC-SA-rate}
\E\brk*{F(x_T) - F^*} \leq c\cdot\prn*{\frac{HB^2}{T^2} + \frac{\sigma B}{\sqrt{T}}}
\end{equation}
To describe the optimal algorithm for the intermittent communication setting, we will first define two distributed variants of AC-SA. 

The first algorithm, which we will refer to as \textbf{Minibatch Accelerated SGD}, implements $R$ iterations of AC-SA using minibatch gradients of size $MK$ \citep[c.f.][]{cotter2011better}. Specifically, the method maintains two iterates $y_r$ and $x_r$ which are shared across all the machines. During each round of communication, each machine computes $K$ independent stochastic estimates of $\nabla F\prn*{\beta_r^{-1} y_r + (1 - \beta_r^{-1})x_r}$; the machines then communicate their minibatches, averaging them together into a larger minibatch of size $MK$, and then they update $y_r$ and $x_r$ according to \eqref{eq:AC-SA-updates}. Because the minibatching reduces the variance of the stochastic gradients by a factor of $MK$, \eqref{eq:AC-SA-rate} implies this method converges at a rate
\begin{equation}\label{eq:MB-rate}
\E\brk*{F(x_R) - F^*} \leq c\cdot\prn*{\frac{HB^2}{R^2} + \frac{\sigma B}{\sqrt{MKR}}}
\end{equation}

The second algorithm, which we will call \textbf{Single-Machine Accelerated SGD}, ``parallelizes'' AC-SA in a different way. In contrast to Minibatch Accelerated SGD, Single-Machine Accelerated SGD simply ignores $M-1$ of the available machines and runs $T=KR$ steps of AC-SA on the remaining one, therefore converging like
\begin{equation}\label{eq:SM-rate}
\E\brk*{F(x_{KR}) - F^*} \leq c\cdot\prn*{\frac{HB^2}{K^2R^2} + \frac{\sigma B}{\sqrt{KR}}}
\end{equation}

From here, we point out that lower bound in Theorem \ref{thm:lower-bound} is equal (up to $\log$ factors) to the minimum of \eqref{eq:MB-rate} and \eqref{eq:SM-rate}. Furthermore, one can determine which of these algorithms achieves the minimum based on the problem parameters:
\begin{restatable}{theorem}{optalgorithms}\label{thm:opt-algorithms}
For any $H,B,\sigma,K,R,M > 0$, the algorithm which returns the output of Minibatch Accelerated SGD when $K \leq \frac{\sigma^2 R^3}{H^2B^2}$ and returns the output of Single-Machine Accelerated SGD when $K > \frac{\sigma^2 R^3}{H^2B^2}$ is optimal up to a factor of $O(\log^2 M)$. 
\end{restatable}
This optimal algorithm is computationally efficient and requires no significant overhead. Each machine needs to store only a constant number of vectors, it performs only a constant number of vector additions for each stochastic gradient oracle access, and it communicates just one vector per round. Therefore, the total storage complexity is $O(d)$ per machine, the sequential runtime complexity is $O(KR\cdot d)$, and the total communication complexity is $O(MR\cdot d)$. In fact, the communication complexity is $0$ when Single-Machine Accelerated SGD is used. Therefore, we do not expect a substantially better algorithm from the standpoint of computational efficiency either.

In light of Theorem \ref{thm:opt-algorithms} and the second $\min$ term in Theorem \ref{thm:lower-bound}, we see that algorithms in this setting are offered the following dilemma: they may either attain the optimal statistical rate $\sigma B / \sqrt{MKR}$ but suffer an optimization rate $HB^2 / (R^2 \log^2 M)$ that does not benefit from $K$ at all, or they may attain the optimal optimization rate of $HB^2 / (K^2R^2)$ but suffer a statistical rate $\sigma B / \sqrt{KR}$ as if only single machine were available. In this sense, there is a very real dichotomy between exploiting parallelism and leveraging local computation.

The main shortcoming of the optimal algorithm is the need to know the problem parameters $H$, $B$, and $\sigma$ to implement it. However, knowledge of these parameters is anyway needed in order to choose the stepsizes for AC-SA, and we are not aware of accelerated variants of SGD that can be implemented without knowing them, even in the sequential setting. This algorithm is also somewhat unnatural because of the hard switch between Minibatch and Single-Machine Accelerated SGD. It would be nice, if only aesthetically, to have an algorithm that more naturally transitions from the Minibatch to the Single-Machine rate. Accelerated Local SGD \citep{yuan2020federated} or something similar is a contender for such an algorithm, although it is unclear whether or not this method can match the optimal rate in all regimes. Local SGD methods can also be augmented by using two stepsizes---a smaller, conservative stepsize for the local updates between communications, and a larger, aggressive stepsize when the local updates are aggregated---this two-stepsize approach allows for interpolation between Minibatch-like and Single-Machine-like behavior, and could be used to design a more ``natural'' optimal algorithm \citep[see Section 6,][]{woodworth2020minibatch}.

Finally, the upper and lower bounds match up to a factor of $\log^2 M$. While this is generally a minor gap, it does raise the question of what the optimal error would be in a massively parallel regime where exponentially many machines are available. In this case, it is conceivable that a brute-force approach might be available that could converge at the rate $1/(R\log M)^2$ in a certain regime, as is suggested by the lower bound, improving over the $1/R^2$ rate achieved by Minibatch Accelerated SGD. Nevertheless, it is not obvious how this could be achieved without any dependence on the dimension and related work by \citet{duchi18a} suggests that such a rate would not be possible without depending on the dimension. We therefore conjecture that the $\log^2 M$ factor can be removed from the lower bound.

\section{Better than Optimal: Breaking the Lower Bound} \label{sec:better-methods}

Perhaps the most important use of a lower bound is in understanding how to break it.  Instead of viewing the lower bound as telling us to give up any hope of improving over the na\"ive optimal method in Section \ref{sec:optimal-alg}, we should view it as informing us about possible means of making progress.  

One way to break our lower bound is by introducing additional assumptions that are not satisfied by the hard instance. These assumptions could then be used to establish when and how some alternate method improves over the ``optimal'' method in Section \ref{sec:optimal-alg}.  Several methods, which operate within the intermittent communication framework of Section \ref{sec:setting-and-notation}, have been shown to be better than the ``optimal algorithm'' in practice {\em for specific instances}. However, attempts to demonstrate the benefit of these methods theoretically have so far failed, and we now understand why.  In order to understand such benefits, we {\em must} introduce additional assumptions, and ask not ``is this alternate method better'' but rather ``under what assumption is this alternate method better?''   Below we suggest possible additional assumptions, including ones that have appeared in recent analysis and also other plausible assumptions one could rely on.

Another way to break the lower bound is by considering algorithms that go beyond the stochastic oracle framework of Section \ref{sec:setting-and-notation}, utilizing more powerful oracles that nevertheless could be equally easy to implement.  Understanding the lower bound can inform us of what type of such extensions might be useful, thus guiding development of novel types of optimization algorithms.

\subsection{Relying on a Bounded Third Derivative}
% \natiinline{In fact, I would give their upper bound before Theorem 4, and then after stating the Theorem, discuss tightness as Ohad suggested.  I'd also rewrite the begining of this sub-section to be centered more around Yuan and Ma and less around your own work (there are a bit too many Woodworth et al citations in the paper).  You can and should mention the quadratic result, but more briefly, and more directly as motivation for Yuan and Ma.  E.g.: 
% As discussed earlier, the lower bound doesn't hold for quadratics, and in fact for quadratic objectives the min-max rate is achieved by local SGD and is much better: (bound).  This suggests "near quadratics" would also be easy, and one way to quantify this is via a bound on the third derivative, is by bounding the third derivatives, as done by Yuan, who showed....  You can also include in this Section a discussion of how this is different from 1st order sequential optimization, where 3rd order smoothness doesn't help, and that this is an important distinction between sequential and distributed highlighted by the lower bound.}

As we have mentioned, Theorem \ref{thm:lower-bound} does not hold in the special case of quadratic objectives of the form $Q(x) = \frac{1}{2}x^\top A x + b^\top x$ for p.s.d.~$A$, e.g.~least squares problems, in which case the min-max rate is much better, and Accelerated Local SGD achieves:
\begin{equation}
\E Q(\hat{x}) - Q^* \leq c\cdot\prn*{\frac{HB^2}{K^2R^2} + \frac{\sigma B}{\sqrt{MKR}}}
\end{equation}
Since improvement over the lower bound is possible when the objective is \emph{exactly} quadratic, it stands to reason that similar improvement should be possible when the objective is sufficiently \emph{close} to quadratic. Indeed, \citet{yuan2020federated} analyze another accelerated variant of Local SGD in the smooth, convex setting with the additional assumption that the Hessian $\nabla^2 F(x)$ is $\alpha$-Lipschitz. Their algorithm converges at a rate
\begin{equation}\label{eq:yuan-bound}
\E F(\hat{x}) - F^* \leq \tilde{O}\prn*{\frac{HB^2}{KR^2} + \frac{\sigma B}{\sqrt{MKR}} + \prn*{\frac{H\sigma^2 B^4}{MKR^3}}^{1/3} + \prn*{\frac{\alpha \sigma^2 B^5}{R^4 K}}^{1/3}}
\end{equation}
This \emph{can} improve over the lower bound in Theorem \ref{thm:lower-bound} in certain parameter regimes, for instance, \eqref{eq:yuan-bound} is better if
\begin{equation}
\frac{H^2B^2}{\sigma^2} \leq \frac{R^3}{MK} \qquad\textrm{and}\qquad \alpha \leq \tilde{O}\prn*{\min\crl*{\frac{\sigma R^{5/2}}{B^2K^{1/2}},\, \frac{H^3BK}{\sigma^2 R^2 \log^6 M}}}
\end{equation}
However, \citeauthor{yuan2020federated}'s guarantee does not \emph{always} improve over the lower bound, and it is not completely clear to what extent further improvement over their algorithm might be possible. In an effort to understand when it may or may not be possible to improve, we extend our lower bound to the case where $\nabla^2 F$ is $\alpha$-Lipschitz:
\begin{restatable}{theorem}{bdedthird}\label{thm:bounded-third}
For any $H,B,\sigma,Q,K,R > 0$ and any $M \geq 2$, there exists a convex, $H$-smooth objective $F$ with $\nrm{x^*}\leq B$ and with $\nabla^2 F$ being $Q$-Lipschitz with respect to the L2 norm, and a stochastic gradient oracle $g$ with $\E\nrm{g(x) - \nabla F(x)}^2 \leq \sigma^2$ for all $x$, such that with probability at least $\frac{1}{2}$ all of the oracle queries $\crl{x^m_{k,r}}$ made by any distributed-zero-respecting intermittent communication algorithm (see Definition \ref{def:zr} in Appendix \ref{app:thm5-proofs}) will have suboptimality
\begin{multline*}
\min_{m,k,r}F(x^m_{k,r}) - F^*
\geq  c\cdot\bigg[\frac{HB^2}{K^2R^2} + \min\crl*{\frac{\sigma B}{\sqrt{MKR}},HB^2} \\+ \min\crl*{\frac{HB^2}{R^2 \log^2 M},\, \frac{\sqrt{Q \sigma} B^2}{K^{1/4}R^2\log^{7/4} M},\, \frac{\sigma B}{\sqrt{KR}}}\bigg]
\end{multline*}
\end{restatable}
We prove this lower bound in Appendix \ref{app:thm5-proofs} using the same construction \eqref{eq:def-F-main} as we used for Theorem \ref{thm:lower-bound}, but using the parameter $\beta$ to control the third derivative of $F$.
This lower bound does not match the guarantee of \citeauthor{yuan2020federated}'s algorithm, so it does not resolve the min-max complexity. However, there is reason to suspect that the lower bound is closer to the min-max rate, at least in certain regimes. For instance, when $Q$ is taken to zero, i.e.~the objective becomes quadratic, we know that Theorem \ref{thm:bounded-third} is tight while \eqref{eq:yuan-bound} can be larger by a factor of $K$. For that reason, we suspect that \eqref{eq:yuan-bound} is suboptimal, but further analysis will be needed. At any rate, our lower bound does establish that there is a limit to the utility of assuming a Lipschitz Hessian. Specifically, there can be no advantage over the optimal algorithm from Section \ref{sec:optimal-alg} once $Q \geq O\prn*{\max\crl*{\frac{H^2\sqrt{K}}{\sigma\log^{1/2}M},\, \frac{\sigma R^3 \log^{7/2}M}{B^2\sqrt{K}}}}$.

Theorem \ref{thm:bounded-third} and \citeauthor{yuan2020federated}'s algorithm also highlight a substantial qualitative difference between distributed and sequential optimization: in the sequential setting, there is never any advantage to assuming that the objective is close to quadratic. In fact, worst-case instances for sequential optimization are exactly quadratic \citep{nemirovskyyudin1983,nesterov2004introductory,simchowitz2018randomized}.

% You can also include in this Section a discussion of how this is different from 1st order sequential optimization, where 3rd order smoothness doesn't help, and that this is an important distinction between sequential and distributed highlighted by the lower bound.

% Unlike before, the lower bound in Theorem \ref{thm:bounded-third} does not match the guarantee of any optimization algorithm that we are aware of, including \citeauthor{yuan2020federated}'s. Nevertheless, we can conclude that once , no improvement over Minibatch and Single-Machine Accelerated SGD is possible.

% \os{to make things clearer, would be good to quote Yuan and Ma's upper bound, and perhaps conjecture whether the lower bound is tight, or the upper bound or something in between.}

Beyond requiring that the Hessian be Lipschitz, there are other ways of measuring an objective's closeness to a quadratic. Two notable examples are self-concordance \citep{nesterov1998introductory} and quasi-self-concordance \citep{bach2010self}, which bound the third derivative of $F$ in terms of the second derivative: we say that $F$ is $Q$-self-concordant when for all $x,v$, $f(t) = F(x + tv)$ satisfies $\abs{f'''(t)} \leq 2Q f''(t)^{3/2}$ and we say it is $Q$-quasi-self-concordant if $\abs{f'''(t)} \leq Q f''(t)$. There has been recent interest in such objectives \citep{bach2010self,zhang2015disco,karimireddy2018global,carmon2020acceleration} which arise e.g.~in logistic regression problems. In Appendix \ref{app:thm5-proofs}, we extend the lower bound in Theorem \ref{thm:bounded-third} to these settings.

\subsection{Statistical Learning Setting: Assumptions on Components}
Stochastic optimization commonly arises in the context of statistical learning, where the goal is to minimize the expected loss with respect to a model's parameters. In this case, the objective can be written $F(x) = \E_{z\sim\mc{D}} f(x;z)$, where $z\sim\mc{D}$ represents data drawn i.i.d.~from an unknown distribution, and the ``components'' $f(x;z)$ represent the loss of the model parametrized by $x$ on the example $z$. 

In the setting of Theorem \ref{thm:lower-bound}, we only place restrictions on the $F$ itself, and on the first and second moments of $g$. However, in the statistical learning setting, it is often natural to assume that the loss function $f(\cdot;z)$ itself satisfies particular properties \emph{for each $z$ individually}. For instance, in our setting we might assume $f$ is convex and smooth and furthermore that the gradient oracle is given by $g(x) = \nabla f(x;z)$ for an i.i.d.~$z\sim\mc{D}$. This is a non-trivial restriction on the stochastic gradient oracle, and it is conceivable that this property could be leveraged to design and analyze a method that converges faster than the lower bound in Theorem \ref{thm:lower-bound} would allow. 

In particular, the specific stochastic gradient oracle \eqref{eq:def-g-main} used to prove Theorem \ref{thm:lower-bound} \emph{cannot} be written as the gradient of a random smooth function. In this sense, the lower bound construction is somewhat ``unnatural,'' however, we are not aware of any analysis that meaningfully\footnote{Numerous papers assume that $F(x) = \E_{z\sim\mc{D}}f(x;z)$ and $g = \nabla f(\cdot;z)$ for some smooth, convex $f$ \citep[e.g.][]{bottou2018optimization,nguyen2019new,koloskova2020unified,woodworth2020minibatch}. Nevertheless, the \emph{purpose} of this assumption is to bound $\E\nrm{g(x)}^2$ or $\E\nrm{g(x) - \nabla F(x)}^2$ in terms of $\sigma_*^2 = \E\nrm{g(x^*)}^2$. In other words, one could prove the same guarantees in the setting of Theorem \ref{thm:lower-bound} with the additional constraint of the form $\E\nrm{g(x)}^2 \leq \sigma_*^2 + \Gamma \nrm{x-x^*}^2$ for some parameter $\Gamma$. Since the variance of the gradient oracle in our lower bound construction is bounded everywhere by a constant $\sigma^2$, it therefore applies to these analyses.} exploits the fact that $g = \nabla f(\cdot;z)$. An interesting question is whether such an assumption can be used to prove a better convergence guarantee, or whether Theorem \ref{thm:lower-bound} can be proven using a stochastic gradient oracle that obeys this constraint.

% \natinote{It would be good to clarify that just being able to write it as a stochastic objective is not an assumptions.  Rather, it allows us to talk about assumptions about the components, such as smoothness.  I would also write the end more hopeful and constructive, saying that it's strange the lower bound is on something non-natural, but the Q is how to actually leverage the smoothness of the components in the analysis, since this is something no previous analysis has done.}

\subsection{Statistical Learning Setting: Repeated Access to Components}\label{sec:repeated-accesses}
% \natinote{I would separate this, to emphasize there are two different types of things here: a new assumption vs stronger oracle.  Also, expand this subsection.  Mention double-accesses (as in John 0-order work and the recent paper with Yair and Dylan), subsampling with component identity knowledge and access to specified components.  Mention known gaps when these are allowed in other models, including 0-order and finite sum.  Then elaborate a tiny bit more on DSVRG, citing also other papers that it discuss it or similar methods}

In the statistical learning setting, it is also natural to consider algorithms that can evaluate the gradient at multiple points for the same datum $z$. Specifically, allowing the algorithm access to a pool of samples $z_1,\dots,z_N$ drawn i.i.d.~from $\mc{D}$ and to compute $\nabla f(x;z)$ for any chosen $x$ and $z_n$ opens up additional possibilities. Indeed, \citet{arjevani2019lower} showed that multiple---even just two---accesses to each component enables substantially faster convergence ($T^{-1/3}$ vs.~$T^{-1/4}$) in sequential stochastic non-convex optimization. Similar results have been shown for zeroth-order and bandit convex optimization \citep{agarwal2010optimal,duchi2015optimal,shamir2017optimal,nesterov2017random}, where accessing each component twice allows for a quadratic improvement in the dimension-dependence. 

In sequential smooth convex optimization, if $F$ has ``finite-sum'' structure (i.e.~$\mc{D}$ is the uniform distribution on $\crl{1,\dots,N}$), then allowing the algorithm to pick a component and access it multiple times opens the door to variance-reduction techniques like SVRG \citep{johnson2013accelerating}. These methods have updates of the form:
\begin{equation}
x_{t+1} = x_t - \eta_t \prn*{\nabla f(x_t;z_t) - \nabla f(\tilde{x};z_t) + \nabla F(\tilde{x})}
\end{equation}
Computing this update therefore requires evaluating the gradient of $f(x;z_t)$ at two different points, which necessitates multiple accesses to a chosen component. This stronger oracle access allows faster rates compared with a single-access oracle \citep[see discussion in, e.g.,][]{arjevani2020complexity}. 

Most relevantly, in the intermittent communication setting, distributed variants of SVRG are able to improve over the lower bound in Theorem \ref{thm:lower-bound} \citep{wang2017memory,lee2017distributed,shamir2016without,woodworth2018graph}. For example, in the intermittent communication setting when $f$ is $H$-smooth and $L$-Lipschitz, and where the algorithm can access each component multiple times, \citeauthor{woodworth2018graph} show that using distributed SVRG to optimize an empirical objective composed of suitably many samples is able to achieve convergence at the rate
\begin{equation}
\E F(\hat{x}) - F^* \leq c\cdot\prn*{\prn*{\frac{HB^2}{RK} + \frac{LB}{\sqrt{MKR}}}\log\frac{MKR}{LB}}
\end{equation}
While this guarantee (necessarily!) holds in a different setting than Theorem \ref{thm:lower-bound}, the Lipschitz bound $L$ is generally analogous to the standard deviation of the stochastic gradient variance, $\sigma$ (indeed, $L$ is an upper bound on $\sigma$). With this in mind, this distributed SVRG algorithm can beat the lower bound in Theorem \ref{thm:lower-bound} when $\sigma$, $L$, and $K$ are sufficiently large.

% Then \citet{woodworth2018graph} showed that using distributed SVRG to optimize an empirical objective composed of $O(MKR)$ samples from $\mc{D}$ can achieve faster convergence than Theorem \ref{thm:lower-bound} allows as long as $M$ and $K$ are sufficiently large. 

% \os{need to specify the bound one gets, and in what sense one uses $\nabla f$ at a chosen $z$ -- namely, to query the gradient of $f(x;z)$ at more than one $x$}

\subsection{Higher Order and Other Stronger Oracles}

% \os{to me, ``higher-order'' means using higher-order derivatives, but a prox-oracle or hessian-vector oracle does not quite fall in this framework. It's more about ``stronger'' oracles. But the SVRG from the previous subsection is also a stronger oracle. Maybe just have one subsection with discussion of various stronger oracles?}

Yet another avenue for improved algorithms in the intermittent communication setting is to use stronger stochastic oracles. For instance, a stochastic second-order oracle that estimates $\nabla^2 F(x)$ \citep{hendrikx2020statistically} or a stochastic Hessian-vector product oracle that estimates $\nabla^2 F(x) v$ given a vector $v$, which can typically be computed as efficiently as stochastic gradients. In the statistical learning setting, some recent work also considers a stochastic prox oracle which returns $\argmin_y f(y;z) + \frac{1}{2}\nrm{x-y}^2$ \citep{wang2017memory,chadha2021accelerated}.
% \natinote{Mention DiSCO as a method using Hessian-vector-Products?}
%The bulk of the research on distributed convex optimization focuses on first-order methods, which are typically more practical, however 

As an example, a stochastic Hessian-vector product oracle, in conjunction with a stochastic gradient oracle can be used to efficiently implement a distributed Newton algorithm. Specifically, the Newton update $x_{t+1} = x_t - \eta_t \nabla^2 F(x_t)^{-1} \nabla F(x_t)$ can be rewritten as
\begin{equation}
x_{t+1} = x_t + \eta_t \argmin_{y} \crl*{\frac{1}{2} y^\top \nabla^2 F(x_t) y + \nabla F(x_t)^\top y}
\end{equation}
That is, each update can be viewed as the solution to a quadratic optimization problem, and its stochastic gradients can be computed using stochastic Hessian-vector and gradient access to $F$. The DiSCO algorithm \citep{zhang2015disco} uses distributed preconditioned conjugate gradient descent to find an approximate Newton step. Alternatively, as previously discussed, this quadratic can be minimized to high accuracy using a single round of communication using Accelerated Local SGD. Under suitable assumptions (e.g., that $F$ is convex, smooth and self-concordant), this algorithm may converge substantially faster than the lower bounds in Theorems \ref{thm:lower-bound} and \ref{thm:bounded-third} would allow for first-order methods.

\paragraph{Differences from Sequential Setting:} Interestingly, in the sequential setting there is no benefit to using stochastic Hessian-vector products over and above what can be achieved using just a stochastic gradient oracle. This is because the worst-case instances are simply quadratic, in which case Hessian-vector products and gradients are essentially equivalent. This adds to a list of structures that facilitate distributed optimization while being essentially useless in the sequential setting. Likewise, objectives being quadratic or near-quadratic facilitates distributed optimization but does not help sequential algorithms since, again, the hard instances for sequential optimization are already quadratic. Furthermore, accessing a statistical learning gradient oracle $\nabla f(\cdot;z)$ multiple times can allow for faster distributed algorithms---e.g.~distributed SVRG or using the stochastic gradients to implement stochastic Hessian-vector products via finite-differencing---but it does not generally help in the sequential case without further assumptions (like the problem having finite-sum structure).

% \natinote{Do we want to be more specific here?  Or maybe not...  We can also mention relationship between HvP and repeated access, as well as that HvP doesn't help for sequential, but does for distributed} % Blake: I think we save all this for the ICML paper.

\subsection{Beyond Single-Sample Oracles}
% \natinote{I am feeling less confident about this Section now.  What methods other than DANE solve the prox problem?  I thought there were a bunch.}

Another class of distributed optimization algorithms, which includes ADMM \citep{boyd2011distributed} and DANE \citep{shamir2014communication}\removed{, DiSCO \citep{zhang2015disco} and AIDE \citep{reddi2016aide}}, involve solving an optimization problem on each machine $m=1..M$ at each round $r$ of the form
\begin{equation}\label{eq:mbprox}
    \min_x \frac{1}{K}\sum_{k=1}^K f(x;z_{k,r}^m) + \lambda_{r,m} \norm{x-y_{r,m}}^2, 
\end{equation}
where $f(\cdot;z)$ are components of the objective $F(x)=\E f(x,z)$, and the vectors $y_{r,m}$ and scalars $\lambda_{r,m}$ are chosen by the algorithm.  Although these methods also involve processing $K$ samples,  or components, at each round on each machine, and then communicating between the machines, they are quite distinct from the stochastic optimization algorithms we consider, and fall well outside the ``stochastic optimization with intermittent communication'' model we study.  The main distinction is that in this paper we are focused on stochastic optimization methods, where each oracle access or ``atomic operation'' involves a single ``data point'' $z^m_{k,r}$ (a single component of a stochastic objective), or in our first-order model, a single stochastic gradient estimate, and can generally be performed in time $O(d)$, where $d$ is the dimensionality of $x$.  In particular, each round consists of $K$ separate accesses, and in all the methods we consider, can be implemented in time $O(dK)$.  In contrast, \eqref{eq:mbprox} is a complex optimization problem involving many data points, and cannot be solved with $O(K)$ atomic operations\footnote{It could perhaps be approximately solved using a small number of passes over the $K$ data points, which would put us back within the scope what we study in this paper, but that is not how these method are generally analyzed.}.  This distinction results in the first term of the lower bound in Theorem \ref{thm:lower-bound}, namely the ``optimization term'' $HB^2/(K^2R^2)$, not applying for methods using \eqref{eq:mbprox}.  In particular, even ignoring $M-1$ machines and running the Mini-Batch Prox method \citep{wang2017memory} on a single machine results ensures a suboptimality of 
\begin{equation}\label{eq:mbproxBound}
    \E F(\hat{x}) - F^* \leq O\left(\frac{\sigma B}{\sqrt{KR}}\right), 
\end{equation}
entirely avoiding the first term of Theorem \ref{thm:lower-bound}, and beating the lower bound when $\sigma^2$ is small.  

Another difference is that DANE, as well as other methods which target Empirical Risk Minimization such as DiSCO \citep{zhang2015disco} and AIDE \citep{reddi2016aide}, work on the same batch of $K$ examples per machine in all rounds, i.e.~they use $z^m_{k,r}=z^m_k$ with only $KM$ (rather than $KRM$) random samples $\left\{ z^m_k\right\}_{k\in[K],m\in[M]}$.  In our setup and terminology, they thus require repeated access to components, as discussed above in Section \ref{sec:repeated-accesses}.  Furthermore, since they only use $KM$ samples overall, they cannot guarantee suboptimality better than $\sigma B/\sqrt{KM}$, a factor of $\sqrt{R}$ worse than the second term in Theorem \ref{thm:lower-bound}.
% \natinote{Say something about stochastic variants?}

The Mini-Batch Prox guarantee \eqref{eq:mbproxBound} is disappointing, and suboptimal, once $\sigma^2$  and $M$ are large, and DANE is not optimal, at least when $R$ is large.  Understanding the min-max complexity of the class of methods which solve \eqref{eq:mbprox} at each round on each machine thus remains an important and interesting open problem. We note that lower bounds and the optimality of some of these methods were studied in \citet{arjevani2015communication}, but in a somewhat different, non-statistical distributed setting.

\paragraph{Acknowledgements}
BW is supported by a Google PhD Research Fellowship. This work is also partially supported by NSF-CCF/BSF award 1718970/2016741, and NS is also supported by a Google Faculty Research Award.
% }

\bibliographystyle{plainnat}
\bibliography{bibliography}

\appendix

\section{A Framework for Proving Lower Bounds for Randomized Algorithms}\label{app:framework}

Our lower bounds are based on the idea of showing that in sufficiently high dimension, any randomized algorithms behaves almost as if it is zero-respecting, meaning that its oracle queries are close to the subspace spanned by the previously-seen oracle responses. To formalize this, for a vector $x$, we define its $\alpha$-progress as
\begin{equation}
\prog{\alpha}(x) := \max\crl*{i\,:\, \abs{x_i} > \alpha}
\end{equation}
Our lower bound proceeds by showing that any intermittent communication algorithm will fail to achieve a high amount of progress, even for randomized algorithms that leave the span of previous stochastic gradient queries. We now define the properties that we will require from our oracle construction:
\begin{definition}\label{def:robust-zero-chain-oracle}
A stochastic zeroth- and first-order oracle that returns $(f(x;z),g(x;z))$ for $z \sim\mc{D}$ is an $(\alpha,p,\delta)$-robust-zero-chain if there exist sets $\mc{Z}_0, \mc{Z}_1$ such that 
\begin{enumerate}
\item $\P(z \in \mc{Z}_0 \cup \mc{Z}_1) \geq 1-\delta$
\item $\P(z\in\mc{Z}_0 | z \in \mc{Z}_0 \cup \mc{Z}_1) \geq 1-p$
\item For all $z \in \mc{Z}_0$ and all $x$, $\prog{0}(g(x;z)) \leq \prog{\alpha}(x)$ and there exist functions $f_1,f_2,\dots$ and $g_1,g_2,\dots$ such that
\[
\prog{\alpha}(x) \leq i \implies \begin{cases} 
f(x;z) = f_i(x_1,x_2,\dots,x_i;z) & \\
g(x;z) = g_i(x_1,x_2,\dots,x_i;z) &
\end{cases}
\]
\item For all $z \in \mc{Z}_1$ and all $x$, $\prog{0}(g(x;z)) \leq \prog{\alpha}(x) + 1$ and there exist functions $f_1,f_2,\dots$ and $g_1,g_2,\dots$ such that
\[
\prog{\alpha}(x) \leq i \implies \begin{cases} 
f(x;z) = f_i(x_1,x_2,\dots,x_{i+1};z) & \\
g(x;z) = g_i(x_1,x_2,\dots,x_{i+1};z) &
\end{cases}
\]
\end{enumerate}
\end{definition}
In this section, our main result is to show that any algorithm that interacts with a robust zero chain will have low progress:
\begin{restatable}{lemma}{intermittentcommunicationprogress}\label{lem:intermittent-communication-progress}
For any $\gamma > 0$, let $(f,g)$ be an $(\alpha,p,\delta)$-robust-zero-chain, let $U\in\R^{D\times d}$ be a uniformly random orthogonal matrix with $U^\top U = I_{d\times{}d}$ for $D \geq d + \frac{2\gamma^2}{\alpha^2}\log(32MKRd)$, and let $x^m_{k,r}$ be the $k\mathth$ oracle query on the $m\mathth$ machine during the $r\mathth$ round of communication for an intermittent communication algorithm that interacts with the oracle $(f_U,g_U)$ for $f_U(x;z) = f(U^\top x;z)$ and $g_U(x;z) := U g(U^\top x;z)$. Then if $\max_{m,k,r} \nrm{x^m_{k,r}} \leq \gamma$, then the algorithm's output $\hat{x}$ will have progress at most
\[
\P\prn*{\prog{\alpha}(U^\top \hat{x}) \leq \min\crl*{KR,\ 8KRp + 12R\log M + 12R}} \geq \frac{5}{8} - 2MKR\delta
\]
\end{restatable}

The main ideas leading to Lemma \ref{lem:intermittent-communication-progress} stem from \cite{woodworth16tight} and \cite{carmon2017lower1}, who show that when a random rotation is applied to the objective and the dimension is sufficiently large, every algorithm behaves essentially as if its queries remained in the span of previously seen gradients. In the original arguments, the proof of this claim was extremely complicated and required a great deal of care due to subtleties with conditioning on the stochastic gradient oracle queries. Since then, the argument has gradually be refined and simplified, culminating in \citet{yairthesis} who presents the simplest argument to date. The proof of Lemma \ref{lem:intermittent-communication-progress} therefore resembles the proof of \citep[Proposition 2.4][]{yairthesis}, however, the arguments must be extended to accomodate the intermittent communication setting.

% \subsection{Proof of \ref{lem:intermittent-communication-progress}}\label{app:generic-intermittent-communication-lower-bounds}
To facilitate our proofs, we introduce some notation. Recalling $\mc{Z}_0$ and $\mc{Z}_1$ from Definition \ref{def:robust-zero-chain-oracle}, we define
\begin{equation}
S^m_{k,r} = \min\crl*{d,\ \sum_{k'=1}^{k-1} \indicator{z^m_{k',r} \in \mc{Z}_1} + \sum_{r'=1}^{r-1}\max_{1\leq m'\leq M}\sum_{k'=1}^K \indicator{z^{m'}_{k',r'} \in \mc{Z}_1}}
\end{equation}
We will show that this quantity $S^m_{k,r}$ essentially upper bounds the $\alpha$-progress of the $k\mathth$ query during the $r\mathth$ round of communication on the $m\mathth$ machine. We also define the following ``good events'' where the progress of the algorithm's oracle queries remains small
\begin{align}
\mc{G}^m_{k,r} &= \crl*{\prog{\alpha}(U^\top x^m_{k,r}) \leq S^m_{k,r}} \\
\bar{\mc{G}}^m_{k,r} &= \bigcap_{k'<k}\mc{G}^m_{k',r} \cap \bigcap_{r'<r}\bigcap_{m',k'} \mc{G}^{m'}_{k',r'}
\end{align}
We also define the event
\begin{equation}
Z = \crl*{\forall_{m,k,r}\ z^m_{k,r} \in \mc{Z}_0 \cup \mc{Z}_1}
\end{equation}
Finally, we use
\begin{equation}
U_{\leq i} = \brk*{U_1,U_2,\dots,U_i,0,\dots,0}
\end{equation}
to denote the matrix $U$ with the $(i+1)\mathth$ through $d\mathth$ columns replaced by zeros.

We begin by showing that when the good events $\bar{\mc{G}}^m_{k,r}$ happen, the algorithm's queries are determined by only a subset of the columns of $U$.
\begin{lemma}\label{lem:robust-zero-chain-u-dependence}
Let $(f,g)$ be an $(\alpha,p,\delta)$-robust-zero-chain, let $U \in \R^{D\times d}$ be a uniformly random orthogonal matrix with $U^\top U = I_{d\times d}$, and let $x^m_{k,r}$ be the $k\mathth$ oracle query on the $m\mathth$ machine during the $r\mathth$ round of communication for an intermittent communication algorithm that interacts with the oracle $(f_U,g_U)$ for $f_U(x;z) = f(U^\top x;z)$ and $g_U(x;z) := U g(U^\top x;z)$. Then conditioned on $\mc{S}$---the $\sigma$-algebra generated by $\crl{S^m_{k,r}}_{m,k,r}$---the events $\bar{G}^m_{k,r}$ and $Z$, and the query $x^m_{k,r}$ is a measurable function of $\xi$---the algorithm's random coins---and $U_{\leq S^m_{k,r}}$. Similarly, conditioned on $\mc{S}$, $Z$, and $\bigcap_{m=1}^M\bar{G}^m_{K,R}$, the output of the algorithm, $\hat{x}$ is a measurable function of $\xi$ and $U_{\leq \max_m S^m_{K,R}}$.
\end{lemma}
\begin{proof}
The dependence structure of an intermittent communication algorithm's queries is determined by the communication between the machines. Specifically, the $m\mathth$ machine's $k\mathth$ query during the $r\mathth$ round of communication can only depend on oracle queries that have been communicated to that machine at that time. In other words, $x^m_{k,r}$ may depend on oracle responses
\begin{equation}
\prn*{f_U(x^m_{k',r};z^m_{k',r}), g_U(x^m_{k',r};z^m_{k',r})}
\end{equation}
for $k' < k$---i.e., oracle queries made on that machine earlier in the current round of communication, or on oracle responses
\begin{equation}
\prn*{f_U(x^{m'}_{k',r'},z^{m'}_{k',r'}), g_U(x^{m'}_{k',r'},z^{m'}_{k',r'})}
\end{equation}
for $r' < r$---i.e., oracle queries made on any machine in earlier rounds of communication. Letting $\xi$ denote the random coins of the algorithm, there are thus query functions $\mc{Q}^m_{k,r}$ such that
\begin{multline}
x^m_{k,r} = \mc{Q}^m_{k,r}\bigg(\crl*{x^m_{k',r}, f_U(x^m_{k',r},z^m_{k',r}), g_U(x^m_{k',r},z^m_{k',r}):k'<k} \\
\cup \crl*{x^{m'}_{k',r'}, f_U(x^{m'}_{k',r'},z^{m'}_{k',r'}), g_U(x^{m'}_{k',r'},z^{m'}_{k',r'}):r'<r}, \xi \bigg)
\end{multline}
The key question is: upon which columns of $U$ does the righthand side of this equation depend when we condition on $\mc{S}$ and $\bar{\mc{G}}^m_{k,r}$? To answer this, we note that by Definition \ref{def:robust-zero-chain-oracle}, if $z \in \mc{Z}_0$ then for any $x$ with $\prog{\alpha}(U^\top x) \leq i$, we have 
\begin{align}
f_U(x;z) &= f(U^\top x;z) = f_i(U_{\leq i}^\top x;z) \\
g_U(x;z) &= g(U^\top x;z) = U g_i(U_{\leq i}^\top x;z)
\end{align}
and furthermore, $\prog{0}(g(U^\top x;z))$ so $U g_i(U_{\leq i}^\top x;z) = U_{\leq i} g_i(U_{\leq i}^\top x;z)$. Therefore, for $z \in \mc{Z}_0$ and $\prog{\alpha}(U^\top x) \leq i$, the oracle response $(f_U(x;z),g_U(x;z))$ depends only on $U_{\leq i}$. By essentially the same argument, for $z \in \mc{Z}_1$ and $\prog{\alpha}(U^\top x) \leq i$, the oracle response $(f_U(x;z),g_U(x;z))$ depends only on $U_{\leq i + 1}$.

Therefore, conditioned on the events $Z$ and $\bar{\mc{G}}^m_{k,r}$, for each $m',k',r'$ such that $m'=m$, $r'=r$, and $k'<k$ or $r'<r$, let $i^{m'}_{k',r'} = \min\crl*{d,\ S^{m'}_{k',r'} + \indicator{z^{m'}_{k',r'} \in \mc{Z}_1}} \leq S^m_{k,r}$ then
\begin{align}
f_U(x^{m'}_{k',r'}; z^{m'}_{k',r'}) &= f_{U_{\leq i^{m'}_{k',r'}}}(x^{m'}_{k',r'};z^{m'}_{k',r'}) = f_{U_{\leq S^m_{k,r}}}(x^{m'}_{k',r'};z^{m'}_{k',r'}) \\
g_U(x^{m'}_{k',r'}; z^{m'}_{k',r'}) &= g_{U_{\leq i^{m'}_{k',r'}}}(x^{m'}_{k',r'};z^{m'}_{k',r'}) = g_{U_{\leq S^m_{k,r}}}(x^{m'}_{k',r'};z^{m'}_{k',r'})
\end{align}
We conclude that conditioned on $Z$ and $\bar{\mc{G}}^m_{k,r}$
\begin{multline}
x^m_{k,r} = \mc{Q}^m_{k,r}\bigg(\crl*{x^m_{k',r}, f_{U_{\leq S^m_{k,r}}}(x^m_{k',r},z^m_{k',r}), g_{U_{\leq S^m_{k,r}}}(x^m_{k',r},z^m_{k',r}):k'<k} \\
\cup \crl*{x^{m'}_{k',r'}, f_{U_{\leq S^m_{k,r}}}(x^{m'}_{k',r'},z^{m'}_{k',r'}), g_{U_{\leq S^m_{k,r}}}(x^{m'}_{k',r'},z^{m'}_{k',r'}):r'<r}, \xi \bigg)
\end{multline}
so conditioned on $\mc{S}$, $Z$, and $\bar{\mc{G}}^m_{k,r}$, $x^m_{k,r}$ is a measurable function of $U_{\leq S^m_{k,r}}$ and $\xi$.

We can apply the same argument to the algorithm's output
\begin{equation}
\hat{x} = \hat{X}\prn*{\crl*{x^m_{k,r}, g_{U_{\leq \max_m S^m_{K,R}}}(x^m_{k,r},z^m_{k,r})}_{m,k,r}, \xi}
\end{equation}
which is measurable with respect to $U_{\leq \max_m S^m_{K,R}}$ and $\xi$ conditioned on $\mc{S}$, $Z$, and $\cap_m \bar{\mc{G}}^m_{K,R}$.
\end{proof}

Next, we show a constant-probability upper bound on the random variables $S^m_{k,r}$: 
\begin{lemma}\label{lem:sum-max-binomials}
For any $(\alpha,p,\delta)$-robust-zero-chain,
\[
\P\prn*{\max_{m,k,r} S^m_{k,r} \geq \min\crl*{KR,\ 8KRp + 12R\log M + 12R}\,\middle|\,Z} \leq \frac{1}{4}
\]
\end{lemma}
\begin{proof}
The claim is implied by 
\begin{equation}
\P\prn*{\sum_{r=1}^R\max_{1\leq m \leq M} \sum_{k=1}^K \indicator{z^m_{k,r} \in \mc{Z}_1} \geq \min\crl*{KR,\ 8KRp + 12R\log M + 12R}\,\middle|\,Z} \leq \frac{1}{4}
\end{equation}
The seed, $z$, for the oracle queries are independent, so, conditioned on $Z$ the indicators $\indicator{z^m_{k,r} \in \mc{Z}_1}$ are independent Bernoulli random variables with success probability at most $p$. It follows that for each $m$ and $r$, $\sum_{k=1}^K \indicator{z^m_{k,r} \in \mc{Z}_1}$ are independent $\textrm{Binomial}(K,p)$ random variables. 

Therefore, for each $r$, by the union bound and then the Chernoff bound, for any $c \geq 0$
\begin{equation}
\P\prn*{\max_{1\leq m \leq M} \sum_{k=1}^K \indicator{z^m_{k,r} \in \mc{Z}_1} \geq (1+c)Kp\,\middle|\,Z}
% \leq M\P\prn*{\sum_{k=1}^K \indicator{z^m_{k,r} \in \mc{Z}_1} \geq (1+c)Kp\,\middle|\,Z}
\leq M\exp\prn*{-\frac{c^2 Kp}{2 + c}}
\end{equation}
Furthermore, for any random variable $X \in [0,K]$, $\E X = \int_0^K \P\prn*{X \geq x} dx$. Therefore, for any $\epsilon > 0$
\begin{align}
\E\brk*{\max_{1\leq m \leq M} \sum_{k=1}^K \indicator{z^m_{k,r} \in \mc{Z}_1}\,\middle|\,Z}
&= \int_0^K \P\prn*{\max_{1\leq m \leq M} \sum_{k=1}^K \indicator{z^m_{k,r} \in \mc{Z}_1} \geq x\,\middle|\,Z} dx \\
&= Kp\int_{-1}^{\frac{1-p}{p}} \P\prn*{\max_{1\leq m \leq M} \sum_{k=1}^K \indicator{z^m_{k,r} \in \mc{Z}_1} \geq (1+c)Kp\,\middle|\,Z} dc \\
&\leq (1+\epsilon)Kp + MKp\int_{\epsilon}^{\frac{1-p}{p}} \exp\prn*{-\frac{c^2Kp}{2 + c}} dc \\
&\leq (1+\epsilon)Kp + MKp\int_{\epsilon}^{\infty} \exp\prn*{-\frac{c\epsilon Kp}{2 + \epsilon}} dc \\
&= (1+\epsilon)Kp + \frac{M(2+\epsilon)}{\epsilon}\exp\prn*{-\frac{\epsilon^2Kp}{2+\epsilon}}
\end{align}
For the second line we used the change of variables $x \to (1+c)Kp$. We take $\epsilon = 1 + \frac{3}{Kp}\log M$ to conclude
\begin{equation}
\E\brk*{\max_{1\leq m \leq M} \sum_{k=1}^K \indicator{z^m_{k,r} \in \mc{Z}_1}\,\middle|\,Z}
\leq (1+\epsilon)Kp + \frac{M(2+\epsilon)}{\epsilon}\exp\prn*{-\frac{\epsilon^2Kp}{2+\epsilon}} 
\leq 2Kp + 3\log M + 3
\end{equation}
It follows that 
\begin{equation}
\E\brk*{\sum_{r=1}^R \max_{1\leq m \leq M} \sum_{k=1}^K \indicator{z^m_{k,r} \in \mc{Z}_1}\,\middle|\,Z}
\leq 2KRp + 3R\log M + 3R
\end{equation}
We conclude using Markov's inequality plus the observation that $S^m_{k,r} \leq KR$ for all $m,k,r$.
\end{proof}

Using the previous lemmas, we prove the main result:
\intermittentcommunicationprogress*
\begin{proof}
We begin by conditioning on $Z$ and $\mc{S}$, the $\sigma$-algebra generated by $\crl{S^m_{k,r}}_{m,k,r}$ and bounding
\begin{align}
\P&\prn*{\prog{\alpha}(U^\top \hat{x}) > \max_m S^m_{K,R} \lor \exists_{m,k,r}\ \prog{\alpha}(U^\top x^m_{k,r}) > S^m_{k,r}\,\middle|\,Z,\mc{S}}\nonumber\\
&= \P\prn*{\crl*{\prog{\alpha}(U^\top \hat{x}) > \max_m S^m_{K,R}} \cup \bigcup_{m,k,r} \crl*{\prog{\alpha}(U^\top x^m_{k,r}) > S^m_{k,r}}\,\middle|\,Z,\mc{S}} \\
&= \P\prn*{\crl*{\crl*{\prog{\alpha}(U^\top \hat{x}) > \max_m S^m_{K,R}}\cap \bigcap_{m=1}^M\bar{\mc{G}}^m_{K,R}}  \cup \bigcup_{m,k,r}\crl*{\prog{\alpha}(U^\top x^m_{k,r}) > S^m_{k,r}} \cap \bar{\mc{G}}^m_{k,r}\,\middle|\,Z,\mc{S}} \\
&\leq \P\brk*{\prog{\alpha}(U^\top \hat{x}) > \max_m S^m_{K,R},\ \bigcap_{m=1}^M\bar{\mc{G}}^m_{K,R}\,\middle|\,Z,\mc{S}} + \sum_{m,k,r} \P\prn*{\crl*{\prog{\alpha}(U^\top x^m_{k,r}) > S^m_{k,r}} \cap \bar{\mc{G}}^m_{k,r}\,\middle|\,Z,\mc{S}} \\
% &= \sum_{m,k,r} \P\prn*{\bigcup_{i > S^m_{k,r}}\crl*{\abs*{\inner{U_i}{x^m_{k,r}}} > \alpha} \cap \bar{\mc{G}}^m_{k,r}\,\middle|\,Z,\mc{S}} \\
&\leq \sum_{i > \max_m S^m_{K,R}} \P\prn*{\abs*{\inner{U_i}{\hat{x}}} > \alpha, \bigcap_{m=1}^M \bar{\mc{G}}^m_{K,R}\,\middle|\,Z,\mc{S}} + \sum_{m,k,r} \sum_{i > S^m_{k,r}} \P\prn*{\abs*{\inner{U_i}{x^m_{k,r}}} > \alpha, \bar{\mc{G}}^m_{k,r}\,\middle|\,Z,\mc{S}} 
\end{align}
By Lemma \ref{lem:robust-zero-chain-u-dependence}, there exist measurable functions $\mathsf{A}^m_{k,r}$ and $\mathsf{B}^m_{k,r}$ such that
\begin{equation}
x^m_{k,r} = \mathsf{A}^m_{k,r}(U_{\leq S^m_{k,r}},\xi) \indicator{Z,\bar{\mc{G}}^m_{k,r}} + \mathsf{B}^m_{k,r}(U,\xi) \indicator{\lnot Z \lor \lnot \bar{\mc{G}}^m_{k,r}}
\end{equation}
Therefore, 
\begin{align}
\P&\prn*{\prog{\alpha}(U^\top \hat{x}) > \max_m S^m_{K,R} \lor \exists_{m,k,r}\ \prog{\alpha}(U^\top x^m_{k,r}) > S^m_{k,r}\,\middle|\,Z,\mc{S}}\nonumber\\
&\leq \sum_{i > \max_m S^m_{K,R}} \P\prn*{\abs*{\inner{U_i}{\hat{\mathsf{A}}(U_{\leq \max_m S^m_{K,R}},\xi)}} > \alpha,\ \bigcap_{m=1}^M \bar{\mc{G}}^m_{K,R}\,\middle|\,Z,\mc{S}} \nonumber\\
&\qquad+ \sum_{m,k,r} \sum_{i > S^m_{k,r}} \P\prn*{\abs*{\inner{U_i}{\mathsf{A}^m_{k,r}(U_{\leq S^m_{k,r}},\xi)}} > \alpha, \bar{\mc{G}}^m_{k,r}\,\middle|\,Z,\mc{S}} \\
&\leq \sum_{i > \max_m S^m_{K,R}} \P\prn*{\abs*{\inner{U_i}{\hat{\mathsf{A}}(U_{\leq \max_m S^m_{K,R}},\xi)}} > \alpha\,\middle|\,Z,\mc{S}} \nonumber\\
&\qquad+ \sum_{m,k,r} \sum_{i > S^m_{k,r}} \P\prn*{\abs*{\inner{U_i}{\mathsf{A}^m_{k,r}(U_{\leq S^m_{k,r}},\xi)}} > \alpha\,\middle|\,Z,\mc{S}}
\end{align}
The algorithm's random coins, $\xi$, and the stochastic gradient oracles' random coins, $\crl{z^m_{k,r}}_{m,k,r}$ which determine $Z$ and $\mc{S}$, are independent of the random rotation $U$. Furthermore, for $i > S^m_{k,r}$, $U_i$ conditioned on $U_{\leq S^m_{k,r}}$ is a uniformly random vector on the $(D-S^m_{k,r})$-dimensional unit sphere orthogonal to the range of $U_{\leq S^m_{k,r}}$. Furthermore, by assumption $\nrm{\mathsf{A}^m_{k,r}(U_{\leq S^m_{k,r}},\xi)} \leq \gamma$. Therefore, following \citet{yairthesis} concentration of measure on the sphere implies \citep{ball1997elementary}
\begin{equation}
\P\prn*{\abs*{\inner{U_i}{\mathsf{A}^m_{k,r}(U_{\leq S^m_{k,r}},\xi)}} > \alpha\,\middle|\,Z,\mc{S}}
\leq 2\exp\prn*{-\frac{(D - S^m_{k,r} + 1)\alpha^2}{2\gamma^2}}
\end{equation}
Using the fact that $S^m_{k,r} \leq d$ and $D \geq d + \frac{2\gamma^2}{\alpha^2}\log(32MKRd)$, we conclude that
\begin{align}
\P&\prn*{\prog{\alpha}(U^\top \hat{x}) > \max_m S^m_{K,R} \lor \exists_{m,k,r}\ \prog{\alpha}(U^\top x^m_{k,r}) > S^m_{k,r}\,\middle|\,Z,\mc{S}}\nonumber\\
&\leq 2(MKR + 1) d\exp\prn*{-\frac{(D - d + 1)\alpha^2}{2\gamma^2}} \\
&\leq 2(MKR+1)d\exp\prn*{-\frac{(d + \frac{2\gamma^2}{\alpha^2}\log(32MKRd) - d + 1)\alpha^2}{2\gamma^2}} 
\leq \frac{1}{8} \label{eq:intermittent-communication-progress-lemma-term1}
\end{align}

To complete the proof of the lemma, we note that for $T = \min\crl*{KR,\ 8KRp + 12R\log M + 12R}$ we can upper bound
\begin{align}
&\P\prn*{\prog{\alpha}(U^\top \hat{x}) > T} \nonumber\\
&\leq \P\prn*{\prog{\alpha}(U^\top \hat{x}) > \max_{m,k,r} S^m_{k,r} \lor \exists_{m,k,r} \prog{\alpha}(U^\top x^m_{k,r}) > S^m_{k,r},\ \max_{m,k,r} S^m_{k,r} \leq T} \nonumber\\
&\qquad+ \P\prn*{\max_{m,k,r} S^m_{k,r} > T} \\
&\leq \P\prn*{\prog{\alpha}(U^\top \hat{x}) > \max_{m,k,r} S^m_{k,r} \lor \exists_{m,k,r} \prog{\alpha}(U^\top x^m_{k,r}) > S^m_{k,r}} + \P\prn*{\max_{m,k,r} S^m_{k,r} > T} \\
&= \P\prn*{\prog{\alpha}(U^\top \hat{x}) > \max_{m,k,r} S^m_{k,r} \lor \exists_{m,k,r} \prog{\alpha}(U^\top x^m_{k,r}) > S^m_{k,r}\,\middle|\,Z}\P(Z) \nonumber\\
&\qquad+ \P\prn*{\max_{m,k,r} S^m_{k,r} > T\,\middle|\,Z}\P(Z) \nonumber\\
&+ \P\prn*{\prog{\alpha}(U^\top \hat{x}) > \max_{m,k,r} S^m_{k,r} \lor \exists_{m,k,r} \prog{\alpha}(U^\top x^m_{k,r}) > S^m_{k,r}\,\middle|\,\lnot Z}\P(\lnot Z) \nonumber\\
&\qquad+ \P\prn*{\max_{m,k,r} S^m_{k,r} > T\,\middle|\,\lnot Z}\P(\lnot Z) \\
&\leq \P\prn*{\prog{\alpha}(U^\top \hat{x}) > \max_{m,k,r} S^m_{k,r} \lor \exists_{m,k,r} \prog{\alpha}(U^\top x^m_{k,r}) > S^m_{k,r}\,\middle|\,Z} \nonumber\\
&\qquad+ \P\prn*{\max_{m,k,r} S^m_{k,r} > T\,\middle|\,Z} + 2(1 - \P(Z))
\end{align}
By \eqref{eq:intermittent-communication-progress-lemma-term1}, the first term is bounded by $\frac{1}{8}$, by Lemma \ref{lem:sum-max-binomials} the second term is at most $\frac{1}{4}$, and by the union bound, 
\begin{equation}
\P(Z) \geq (1-\delta)^{MKR} \geq 1 - MKR\delta
\end{equation}
This completes the proof.
\end{proof}

\section{An Extension to Large-Norm Queries}\label{app:large-norm}
In the previous section, we introduce a tool for proving lower bounds for algorithm's whose oracle queries have norm bounded by $\gamma$. Here, we show how to modify the hard instances so that the lower bound applies for any algorithm, even with unboundedly large oracle queries.

\begin{lemma}\label{lem:tildegamma}
The scalar function
\[
\tilde\Gamma(t) = \begin{cases}
0 & t \leq 0 \\
\frac{\int_0^t \exp\prn*{-\frac{1}{s(1-s)}}ds}{\int_0^1 \exp\prn*{-\frac{1}{s(1-s)}}ds} & t \in (0,1) \\
1 & t \geq 1
\end{cases}
\]
is twice differentiable, and for all $t$: $\abs{\tilde\Gamma'(t)} \leq 4$ and $\abs{\tilde\Gamma''(t)} \leq 60$.
\end{lemma}
\begin{proof}
Let $C = \int_0^1 \exp\prn*{-\frac{1}{s(1-s)}}ds$. It is straightforward to confirm numerically that $C \geq \frac{1}{200}$. 
First, we compute the derivatives of $\tilde\Gamma(t)$ for $t \in (0,1)$:
\begin{align}
\tilde\Gamma'(t) &= \frac{1}{C}\exp\prn*{-\frac{1}{t(1-t)}} \\
\tilde\Gamma''(t) &= \frac{1}{C}\exp\prn*{-\frac{1}{t(1-t)}}\frac{1-2t}{t^2(1-t)^2}
\end{align}
Because
\begin{equation}
\lim_{t \nearrow 1} \tilde\Gamma'(t) = \lim_{t \searrow 0} \tilde\Gamma'(t) = \lim_{t \nearrow 1} \tilde\Gamma''(t) = \lim_{t \searrow 0} \tilde\Gamma''(t) = 0
\end{equation}
we conclude that $\tilde\Gamma$ is twice differentiable on $\R$. Furthermore, 
\begin{equation}
\sup_{t\in(0,1)} \abs{\tilde\Gamma'(t)} = \frac{1}{C}\sup_{0 < s \leq \frac{1}{4}}\exp\prn*{-\frac{1}{s}} = \frac{1}{C}e^{-4} \leq 200e^{-4} \leq 4
\end{equation}
and
\begin{align}
\sup_{t\in(0,1)} \abs{\tilde\Gamma''(t)} 
&= \frac{1}{C}\sup_{t\in(0,1)}\exp\prn*{-\frac{1}{t(1-t)}}\abs*{\frac{1-2t}{t^2(1-t)^2}} \\
&\leq \frac{1}{C}\sup_{t\in(0,1)}\exp\prn*{-\frac{1}{t(1-t)}}\frac{1}{t^2(1-t)^2} \\
&= \frac{1}{C}\sup_{s \geq 4} s^2\exp\prn*{-s} 
\end{align}
Finally, $\frac{d}{ds} s^2\exp(-s) = (2-s)s\exp(-s)$, which is negative for all $s \geq 4$, so we conclude
\begin{equation}
\sup_t \abs{\tilde\Gamma''(t)} \leq \frac{1}{C} 4^2e^{-4} \leq 200\cdot 16e^{-4} \leq 60
\end{equation}
This completes the proof.
\end{proof}

\begin{lemma}\label{lem:gamma}
For $\tilde\Gamma$ as defined in Lemma \ref{lem:tildegamma} and any $a,b > 0$, we define
\[
\Gamma(x) = \tilde\Gamma(a(\nrm{x} - b))
\]
This function is twice differentiable, $4a$-Lipschitz, and $60a^2$-smooth.
\end{lemma}
\begin{proof}
We recall from Lemma \ref{lem:tildegamma} that $\abs{\tilde\Gamma'} \leq 4$ and $\abs{\tilde\Gamma''} \leq 60$. We now compute the gradient and Hessian of $\Gamma$:
\begin{align}
\nabla \Gamma(x) &= a\tilde\Gamma'(a(\nrm{x} - b))\frac{x}{\nrm{x}} \\
\nabla^2 \Gamma(x) &= a^2\tilde\Gamma''(a(\nrm{x} - b))\frac{xx^\top}{\nrm{x}^2} + \frac{a\tilde\Gamma'(a(\nrm{x} - b))}{\nrm{x}}\prn*{I - \frac{xx^\top}{\nrm{x}^2}}
\end{align}
We note that since $b > 0$, $x = 0 \implies \tilde\Gamma'(a(\nrm{x} - b)) = \tilde\Gamma''(a(\nrm{x} - b)) = 0$, so the gradient and Hessian are well-defined (and equal to zero) at the point $x=0$.

It is easy to see that for any $x$
\begin{equation}
\nrm*{\nabla \Gamma(x)} = \nrm*{a\tilde\Gamma'(a(\nrm{x} - b))\frac{x}{\nrm{x}}} = a\abs*{\tilde\Gamma'(a(\nrm{x} - b))} \leq 4a
\end{equation}
Similarly, the eigenvalues of $\nabla^2 \Gamma(x)$ are $a^2\tilde\Gamma''(a(\nrm{x} - b))$ (with multiplicity 1) and $\frac{a\tilde\Gamma'(a(\nrm{x} - b))}{\nrm{x}}$ (with multiplicity $\textrm{dimension} - 1$). Therefore,
\begin{equation}
\nrm*{\nabla^2 \Gamma(x)} \leq \max\crl*{a^2\abs{\tilde\Gamma''(a(\nrm{x} - b))},\, \frac{a\abs{\tilde\Gamma'(a(\nrm{x} - b))}}{\nrm{x}}}
\end{equation}
Furthermore, since $\tilde\Gamma'(0) = 0$ and $\abs{\tilde\Gamma''} \leq 60$, we have $\abs{\tilde\Gamma'(a(\nrm{x} - b))} \leq 60\max\crl*{a(\nrm{x} - b),0}$ so
\begin{equation}
\nrm*{\nabla^2 \Gamma(x)} \leq \max\crl*{60a^2,\, \frac{60a^2(\nrm{x} - b)}{\nrm{x}}} = 60a^2
\end{equation}
This completes the proof.
\end{proof}

\begin{lemma}\label{lem:tilde-f-properties}
Let $F$ be convex and $H$-smooth with $F^* = 0$ and $\min_{x:F(x)=F^*}\nrm{x} \leq B$, and let $\Gamma$ be defined as in Lemma \ref{lem:gamma} for 
\[
a = \min\crl*{\frac{1}{408B}, \sqrt{\frac{\sigma^2}{32\rho^2}}} \qquad\textrm{and}\qquad
b = 2B
\]
Then
\[
\tilde{F}(x) := (1-\Gamma(x))F(x) + 42H\max\crl*{0,\ \nrm{x} - B}^2
\]
satisfies the following:
\begin{enumerate}
\item $\tilde{F}$ is convex and $124H$-smooth
\item $\tilde{F}(x) = F(x)$ for all $x$ with $\nrm{x} \leq B$ 
\item $\tilde{F}(x) = 42H\max\crl*{0,\ \nrm{x} - B}^2$ for all $x$ with $\nrm{x} \geq 2B + \max\crl*{408B,\,\sqrt{\frac{32\rho^2}{\sigma^2}}}$.
\item For any $x$, $\tilde{F}(x) - \min_x \tilde{F}(x) \geq F(x) - F^*$.
\end{enumerate}
Furthermore, if $f(x)$ and $g(x)$ are stochastic zeroth-order and first-order oracles for $F$ with variance 
\begin{align*}
\E(f(x) - F(x))^2 &\leq \rho^2 \\
\E\nrm*{g(x) - \nabla F(x)}^2 &\leq \sigma^2
\end{align*}
Then,
\[
\tilde{g}(x) = (1 - \Gamma(x))g(x) - f(x)\nabla \Gamma(x) + 42H\max\crl*{0,\ \nrm{x} - B}\frac{x}{\nrm{x}}
\]
is an unbiased stochastic gradient oracle for $\tilde{F}$ with variance at most
\[
\E \nrm{\tilde{g}(x) - \nabla \tilde{F}(x)}^2 \leq 3\sigma^2
\]
\end{lemma}
\begin{proof}
Let $c = 84H$. We first note that the second derivative of the second term in the definition of $\tilde{F}$ is
\begin{align}
\frac{d^2}{dx^2}&\frac{c}{2}\max\crl*{0,\ \nrm{x} - \frac{b}{2}}^2 \nonumber\\
&= \frac{d}{dx} c\max\crl*{0,\ \nrm{x} - \frac{b}{2}}\frac{x}{\nrm{x}} \\
&= \begin{cases} 
0 & \nrm{x} < \frac{b}{2} \\
c \frac{xx^\top}{\nrm{x}^2} + c\max\crl*{0,\ 1 - \frac{b}{2\nrm{x}}}\prn*{I - \frac{xx^\top}{\nrm{x}^2}} & \nrm{x} \geq \frac{b}{2}
\end{cases}
\end{align}
The eigenvalues of this matrix are $c$ (with multiplicity 1) and $c\max\crl*{0,\ 1 - \frac{b}{2\nrm{x}}} \in [0,c)$ (with multiplicity $\textrm{dimension} - 1$), and therefore this term is convex and $c=84H$-smooth. Furthermore, for $\nrm{x} \geq b$, 
\begin{equation}
c\max\crl*{0,\ 1 - \frac{b}{2\nrm{x}}} \geq \frac{c}{2}
\end{equation}
and therefore this term is actually $\frac{c}{2}=42H$-strongly convex on $\crl{x:\nrm{x}\geq b}$.

From here, we define $\varphi(x) = (1-\Gamma(x))F(x)$ to be the first term of $\tilde{F}$. We will now show that $\varphi$ is convex and $H$-smooth on $\crl{x:\nrm{x}\leq b}$ and is $40H$-smooth on $\crl{x:\nrm{x} \geq b}$ which, together with the previous results, implies that $\tilde{F}$ is convex and $124H$-smooth everywhere.

% First, we compute the derivatives of $\varphi$:
% \begin{align}
% \varphi(x) &= (1-\Gamma(x))F(x) \\
% \nabla \varphi(x) &= (1-\Gamma(x))\nabla F(x) - F(x) \nabla \Gamma(x) \\
% \nabla^2 \varphi(x) &= (1-\Gamma(x))\nabla^2 F(x) - \nabla F(x) \nabla \Gamma(x)^\top - \nabla \Gamma(x)\nabla F(x)^\top - F(x) \nabla^2 \Gamma(x)
% \end{align}

The first piece is simple: for $x \in \crl{x:\nrm{x}\leq b}$, $\Gamma(x) = 0$ so $\varphi(x) = F(x)$, which is convex and $H$-smooth. For the second part, we fix arbitrary $x,y$ with $\nrm{x} \leq \nrm{y}$ and upper bound $\nrm{\nabla \varphi(x) - \nabla \varphi(y)}$. If $1 \leq a(\nrm{x}-b) \leq a(\nrm{y}-b)$, then $(1-\Gamma(x)) = (1-\Gamma(y)) = \Gamma'(x) = \Gamma'(y) = 0$, so $\nabla \varphi(x) = \nabla \varphi(y) = 0$, so $\varphi$ is $0$-smooth on the set $\crl{x:a(\nrm{x}-b) \geq 1}$. Otherwise, when $a(\nrm{x}-b) < 1$ we have
\begin{align}
&\nrm*{\nabla \varphi(x) - \nabla \varphi(y)} \nonumber\\
&= \nrm*{(1-\Gamma(x))\nabla F(x) - F(x) \nabla\Gamma(x) - (1-\Gamma(y))\nabla F(y) + F(y) \nabla\Gamma(y)} \\
&\leq \abs{\Gamma(y)-\Gamma(x)}\nrm*{\nabla F(x)} + \abs{1-\Gamma(y)}\nrm*{\nabla F(x) - \nabla F(y)} \nonumber\\
&\qquad\qquad+ \abs{F(x)}\nrm*{\nabla \Gamma(y) - \nabla\Gamma(x)} + \abs{F(y) - F(x)}\nrm*{\nabla\Gamma(y)} \label{eq:varphi-smooth}
\end{align}
From here, we note that the fact that $F$ is $H$-smooth and has a minimizer with norm at most $B$ implies that $F$ is $H(B + b + \frac{1}{a})$-Lipschitz on the set $\crl{x: a(\nrm{x}-b) < 1} \subseteq \crl{x:\nrm{x-x^*} \leq B + b + \frac{1}{a}}$. Also, from Lemma \ref{lem:gamma}, $\Gamma$ is $4a$-Lipschitz and $60a^2$-smooth. Therefore, from \eqref{eq:varphi-smooth}, we can upper bound:
\begin{align}
&\nrm*{\nabla \varphi(x) - \nabla \varphi(y)} \nonumber\\
% &\leq \prn*{4aH\nrm{x-x^*} + H + 30a^2H\nrm{x-x^*}^2 + 4aH\nrm{y-x^*}\indicatorb{\Gamma(y)<1}}\nrm{x-y} \\
&\leq H\prn*{4a\prn*{B + b + \frac{1}{a}} + 1 + 30a^2\prn*{B + b + \frac{1}{a}}^2 + 4a\prn*{B + b + \frac{1}{a}}}\nrm{x-y} \\
&= H\prn*{270a^2B^2 + 204aB + 39}\nrm{x-y}
\end{align}
With our choice $a \leq \frac{1}{408B}$, this means $\nrm{\nabla \varphi(x) - \nabla \varphi(y)} \leq 40H\nrm{x-y}$, so $\varphi$ is $40H$-smooth on $\crl{x:\nrm{x} \geq b}$. This concludes the proof of points 1 and 2, and the third point follows immediately from the fact that $a(\nrm{x}-b) \geq 1 \implies \Gamma(x) = 1$.

For the fourth point, we begin by observing that since $F$ has a minimizer $x^*$ with $\nrm{x^*} \leq B = \frac{b}{2}$, 
\begin{equation}
\min_x \tilde{F}(x) \leq \tilde{F}(x^*) = (1-\Gamma(x^*))F(x^*) + \frac{c}{2}\max\crl*{0,\ \nrm{x^*} - \frac{b}{2}}^2 = F^*
\end{equation}
Furthermore, since $F \geq F^* = 0$ and $\frac{c}{2}\max\crl*{0,\ \nrm{x^*} - \frac{b}{2}} \geq 0$, we have $\min_x \tilde{F}(x) \geq 0 = F^*$, so $\min_x \tilde{F}(x) = F^* = 0$. Now, all that remains is to show that $\tilde{F}(x) \geq F(x)$. 

Let $x$ be a point with $\nrm{x} \leq b$, then $\Gamma(x) = 0$ so
\begin{equation}
\tilde{F}(x) = F(x) + \frac{c}{2}\max\crl*{0,\ \nrm{x} - \frac{b}{2}}^2 \geq F(x)
\end{equation}
Otherwise, if $x$ has norm $\nrm{x} > b$, we already showed that on the set $\crl{y:\nrm{y} > b}$, $\varphi(y) = (1-\Gamma(y))F(y)$ is $40H=\frac{10c}{21}$-smooth and the second term $\frac{c}{2}\max\crl*{0,\ \nrm{y} - \frac{b}{2}}^2$ is $\frac{c}{2}$-strongly convex. Therefore, $\tilde{F}$ is $\frac{c}{42} = 2H$-strongly convex on $\crl{y:\nrm{y} > b}$. Let $x_b = b\frac{x}{\nrm{x}}$ be the projection of $x$ onto the set $\crl{y:\nrm{y} \leq b}$. By the $H$-smoothness of $F$, we have
\begin{equation}
F(x) \leq F(x_b) + \inner{\nabla F(x_b)}{x - x_b} + \frac{H}{2}\nrm{x - x_b}^2
\end{equation}
On the other hand, by the $2H$-strong convexity of $\tilde{F}$ on $\crl{y:\nrm{y} > b}$, we have
\begin{align}
\tilde{F}(x) 
&\geq \tilde{F}(x_b) + \inner{\nabla \tilde{F}(x_b)}{x-x_b} + H\nrm{x - x_b}^2 \\
&= \tilde{F}(x_b) + \inner{\nabla F(x_b) + c\max\crl*{0,\ \nrm{x} - \frac{b}{2}}\frac{x}{\nrm{x}}}{x-x_b} + H\nrm{x - x_b}^2 \\
&> \tilde{F}(x_b) + \inner{\nabla F(x_b)}{x-x_b} + H\nrm{x - x_b}^2 \\
&\geq F(x) + \frac{H}{2}\nrm{x - x_b}^2 > F(x)
\end{align}

Finally, for the point about the stochastic gradient oracle, it is easy to see that $\tilde{g}$ is unbiased because
\begin{equation}
\E \tilde{g}(x) = (1 - \Gamma(x))\E g(x) - \E f(x)\nabla \Gamma(x) + cH\max\crl*{0,\ \nrm{x} - B}\frac{x}{\nrm{x}} = \nabla\tilde{F}(x)
\end{equation}
For the variance, we also have that 
\begin{align}
\E \nrm*{\tilde{g}(x) - \nabla \tilde{F}(x)}^2
&= \E \nrm*{(1 - \Gamma(x))(g(x) - \nabla F(x)) - (f(x) - F(x))\nabla \Gamma(x)}^2 \\
&\leq 2(1 - \Gamma(x))^2\E \nrm*{g(x) - \nabla F(x)} + 2\E(f(x) - F(x))^2\nrm{\nabla \Gamma(x)}^2 \\
&\leq 2\sigma^2 + 32a^2\rho^2 \\
&\leq 3\sigma^2
\end{align}
This completes the proof.
\end{proof}

\begin{lemma}\label{lem:large-norm}
Fix any $H,B,\sigma,\rho,K,R > 0$ and $M \geq 2$, and let 
\[
\gamma = 2B + \max\crl*{408B,\,\sqrt{\frac{32\rho^2}{\sigma^2}}}
\]
Suppose that there is a family of objectives on $\R^d$ that are convex, $\frac{H}{124}$-smooth, and have a minimizer with norm less than $B$, and that each objective, $F$, in the family is equipped with an oracle $(f(x), g(x))$ such that $\E (f(x),g(x)) = (F(x), \nabla F(x))$, $\E(f(x) - F(x))^2 \leq \frac{\rho^2}{3}$, and $\E \nrm{g(x) - \nabla F(x)}^2 \leq \frac{\sigma^2}{3}$, and suppose that for any intermittent communication algorithm whose oracle queries are guaranteed to have norm less than $\gamma$, their output will have error at least $\E F(\hat{x}) - F^* \geq \epsilon$ for at least one function in the family.

Then, there exists another family of objectives on $\R^d$ that are convex, $H$-smooth, and have a minimizer with norm less than $B$, and each objective, $\tilde{F}$, in the family is equipped with a stochastic gradient oracle $\tilde{g}$ such that $\E \tilde{g}(x) = \nabla \tilde{F}(x)$ and $\E \nrm{\tilde{g}(x)- \nabla\tilde{F}(x)}^2 \leq \sigma^2$, and such that the output of any intermittent communication algorithm (whose oracle queries may have arbitrarily large norm) will have error at least $\E \tilde{F}(\hat{x}) - \tilde{F}^* \geq \epsilon$.
\end{lemma}
\begin{proof}
For each $F,f,g$ in the original family of objectives, we define $\tilde{F},\tilde{g}$ in terms of the function $\Gamma$ as in Lemma \ref{lem:tilde-f-properties}, and we consider the family of all of these $\tilde{F}$'s. As shown in Lemma \ref{lem:tilde-f-properties}, these $\tilde{F}$'s are each convex, $H$-smooth, and have a minimizer with norm at most $B$, and $\tilde{g}$ is an unbiased estimate of $\nabla \tilde{F}$ with variance at most $\sigma^2$.

Now, suppose towards contradiction that some optimization algorithm can ensure that its output satisfies $\E\tilde{F}(\hat{x}) - \tilde{F}^* < \epsilon$ for every objective $\tilde{F}$ in the modified family. By point 4 in Lemma \ref{lem:tilde-f-properties}, this means $\epsilon > \E\tilde{F}(x) - \tilde{F}^* \geq \E F(x) - F^*$, so this algorithm could optimize all of the objectives in the original family to error less than $\epsilon$. Furthermore, although this algorithm might query $\tilde{g}$ at points with norm greater than $\gamma$, these queries could actually be simulated via queries to the oracle $(f(x),g(x))$ using points of norm \emph{less} than $\gamma$. In particular, by point 3 of Lemma \ref{lem:tilde-f-properties}, for $\nrm{x} \geq \gamma$
\begin{equation}
\tilde{g}(x) = \nabla \tilde{F}(x) = 84H\max\crl*{0,\nrm{x} - B}\frac{x}{\nrm{x}}
\end{equation}
and therefore no oracle queries at all are needed to simulate queries to $\tilde{g}$ with large norm. At the same time, for $\nrm{x} < \gamma$, 
\begin{equation}
\tilde{g}(x) = (1-\Gamma(x))g(x) - f(x)\nabla \Gamma(x) + 84H\max\crl*{0,\nrm{x} - B}\frac{x}{\nrm{x}}
\end{equation}
can be calculated using a single query to $(f(x),g(x))$ at $x$ with norm $\nrm{x} < \gamma$. Therefore, the existence of such an algorithm that can optimize $\tilde{F}$ to accuracy less than $\epsilon$ would imply the existence of an algorithm whose queries to $(f(x),g(x))$ all have norm less than $\gamma$ that can optimize $F$ to accuracy less than $\epsilon$, which is a contradiction. We therefore conclude that no such algorithm can exists.
\end{proof}

\section{Proof of Theorem \ref{thm:lower-bound}}\label{app:lower-bound-proof}

Now, we are ready to prove our lower bounds. We construct a hard instance for the lower bound using the scalar functions $\psi:\R\to\R$:
\begin{equation}
\psi(x) = \frac{\sqrt{H}x}{2\beta}\arctan\prn*{\frac{\sqrt{H}\beta x}{2}} - \frac{1}{2\beta^2}\log\prn*{1+\frac{H\beta^2x^2}{4}}
\end{equation}
where $H$ is the parameter of smoothness, and $\beta > 0$ is another parameter that controls the third derivative of $\psi$ which we will set later. The hard instance is then
\begin{equation}\label{eq:def-F}
F(x) = -\psi'(\zeta) x_1 + \psi(x_N) + \sum_{i=1}^{N-1} \psi(x_{i+1} - x_i)
\end{equation}
where $\zeta$ and $N$ are additional parameters that will be chosen later. Lemma \ref{lem:F-properties} below summarizes the relevant properties of $F$

\begin{lemma}\label{lem:F-properties}
For any $H > 0$, $\beta > 0$, $B > 0$, and $N \geq 2$, we set $\zeta = \frac{B}{N^{3/2}}$. Then, $F$ is convex, $H$-smooth, and $\nabla^2 F(x)$ is $\frac{H^{3/2}\beta}{3}$-Lipschitz; there exists $x^* \in \argmin_x F(x)$ with $\nrm{x^*} \leq B$; and for any $x$ with $\prog{\alpha}(x) \leq \frac{N}{2}$ for $\alpha \leq \min\crl*{\frac{N}{12\beta\sqrt{H}},\,\frac{\sqrt{H}\beta B^2}{64N^2}}$,
\[
F(x) - F^* \geq \begin{cases}
\frac{N}{12\beta^2} & \beta^2 > \frac{4N^3}{HB^2} \\
\frac{HB^2}{64N^2} & \beta^2 \leq \frac{4N^3}{HB^2} 
\end{cases}
\]
\end{lemma}
\begin{proof}
First, we note that $0 \leq \psi''(x) = \frac{H}{4 + H\beta^2 x^2} \leq \frac{H}{4}$. Therefore, $F$ is the sum of convex functions and is thus convex itself. We now compute the Hessian of $F$:
\begin{equation}
\nabla^2 F(x) = \psi''(x_N)e_Ne_N^\top + \sum_{i=1}^{N-1}\psi''(x_{i+1} - x_i)(e_{i+1}-e_i)(e_{i+1}-e_i)^\top
\end{equation}
Therefore, for any $u \in \R$,
\begin{align}
u^\top \nabla^2 F(x) u 
&\leq \psi''(x_N) u_N^2 + \sum_{i=1}^{N-1}\psi''(x_{i+1} - x_i)(u_{i+1}-u_i)^2 \\
&\leq \frac{H}{4}\brk*{u_N^2 + \sum_{i=1}^{N-1}2u_{i+1}^2 + 2u_i^2} \\
&\leq H\nrm{u}^2
\end{align}
We conclude that $\nabla^2 F(x) \preceq H\cdot I$ and thus $F$ is $H$-smooth.

Next, we compute the tensor of 3rd derivatives of $F$:
\begin{equation}
\nabla^3 F(x) = \psi'''(x_N)e_N^{\otimes 3} + \sum_{i=1}^{N-1}\psi'''(x_{i+1} - x_i)(e_{i+1}-e_i)^{\otimes 3}
\end{equation}
where
\begin{equation}
\psi'''(x) = \frac{-2H^2\beta^2 x}{(4 + H\beta^2 x^2)^2}
\end{equation}
Therefore, for any $u,v\in\R$,
\begin{equation}
\abs*{\nabla^3 F(x)[u,u,v]} \leq \abs*{\psi'''(x_N)u_N^2v_N} + \sum_{i=1}^{N-1}\abs*{\psi'''(x_{i+1} - x_i)(u_{i+1}-u_i)^2(v_{i+1}-v_i)}
\end{equation}
We can bound this in several different ways using Lemma \ref{lem:technical-1}:
\begin{align}
\abs{\psi'''(x)} &\leq \frac{H^{3/2}\beta}{12} \\
\abs{\psi'''(x)} &\leq 2\beta\psi''(x)^{3/2} \\
\abs{\psi'''(x)} &\leq \frac{\sqrt{H}\beta}{2}\psi''(x)
\end{align}
Therefore,
\begin{align}
\abs*{\nabla^3 F(x)[u,u,v]} 
&\leq \abs*{\psi'''(x_N)u_N^2v_N} + \sum_{i=1}^{N-1}\abs*{\psi'''(x_{i+1} - x_i)(u_{i+1}-u_i)^2(v_{i+1}-v_i)} \\
&\leq \sup_x\abs{\psi'''(x)}\nrm{v}_\infty\prn*{\abs*{u_N}^2 + \sum_{i=1}^{N-1}(u_{i+1}-u_i)^2} \\
&\leq \sup_x\abs{\psi'''(x)}4\nrm*{u}^2\nrm{v}\label{eq:hessian-lipschitz-constant}
\end{align}
Above, we used the H\"older inequality $\sum_i \abs{a_ib_i} \leq \nrm{a}_1 \nrm{b}_\infty$. We conclude that $F$ is $\beta$-self-concordant. We conclude that $\nabla^2 F(x)$ is $4\sup_x\abs{\psi'''(x)}$-Lipschitz. To conclude, we upper bound
\begin{equation}
\sup_x \abs{\psi'''(x)} = \sup_x \abs*{\frac{-2H^2\beta^2 x}{(4 + H\beta^2 x^2)^2}}
\end{equation}
To do so, we maximize the simpler function $x\mapsto \frac{x}{(1 + x^2)^2}$. We note that
\begin{gather}
\frac{d}{dx} \frac{x}{(1 + x^2)^2} = \frac{1 - 3x^2}{(1+x^2)^3} \\
\frac{d^2}{dx^2} \frac{x}{(1 + x^2)^2} = \frac{12x(x^2-1)}{(1+x^2)^4}
\end{gather}
Therefore, the derivative is zero at $\pm 1/\sqrt{3}$ and as $x \to \pm \infty$ and the second derivative is negative only for $+ 1/\sqrt{3}$, and $\lim_{x\to\pm\infty} \frac{x}{(1 + x^2)^2} = 0$. Therefore, we conclude that
\begin{equation}
\sup_x \frac{x}{(1 + x^2)^2} = \sup_x \frac{\abs{x}}{(1 + x^2)^2}
= \frac{\sqrt{\frac{1}{3}}}{(1 + \frac{1}{3})^2} 
= \frac{3\sqrt{3}}{16}
\end{equation}
By rescaling, we conclude that
\begin{equation}
\sup_x \abs*{\psi'''(x)}
= \sup_x \frac{2H^2\beta^2\abs{x}}{(4 + H\beta^2x^2)^2}
= \frac{H^{3/2}\beta}{4}\sup_x \frac{\abs*{\frac{\sqrt{H}\beta x}{2}}}{\prn*{1 + \prn*{\frac{\sqrt{H}\beta x}{2}}^2}^2}
= \frac{3\sqrt{3}H^{3/2}\beta}{64}
< \frac{H^{3/2}\beta}{12}
\end{equation}
Combining this with \eqref{eq:hessian-lipschitz-constant} completes the upper bound on the Hessian Lipschitz parameter.

We now bound the norm of the minimizer of $F$. The first-order optimality condition $\nabla F(x^*) = 0$ indicates
\begin{equation}
\begin{aligned}
[\nabla F(x^*)]_1 = 0 &= \psi'(-\zeta) - \psi'(x^*_2 - x^*_1) \\
[\nabla F(x^*)]_i = 0 &= \psi'(x^*_i - x^*_{i-1}) - \psi'(x^*_{i+1} - x^*_i) \qquad 1 < i < N \\
[\nabla F(x^*)]_N = 0 &= \psi'(x^*_N - x^*_{N-1}) + \psi'(x^*_N)
\end{aligned}
\end{equation}
Because $\psi'(x) = \arctan(x)$ is invertible on its range, we conclude that $x^*_i - x^*_{i+1} = \zeta$ for $i < N$, and $x^*_N = \zeta$. So, the following point minimizes $F$:
\begin{equation}
x^* = \zeta\sum_{i=1}^N (N-i+1)e_i
\end{equation}
This point has norm
\begin{equation}
\nrm*{x^*}^2
= \zeta^2\sum_{i=1}^N (N-i+1)^2 
= \frac{\zeta^2}{6}\prn*{2N^3 + 3N^2 + N} \leq \zeta^2N^3
\end{equation}
Therefore, setting $\zeta^2 = \frac{B^2}{N^3}$ ensures the existence of a minimizer with norm less than $B$. 

At this point,
\begin{equation}
F(x^*) = -N\zeta \psi'(\zeta) + \psi(\zeta) + \sum_{i=1}^N \psi(-\zeta) = N(\psi(\zeta) - \zeta \psi'(\zeta))
\end{equation}
Finally, by Jensen's inequality and the convexity of $\psi$, for any $I$
\begin{equation}
\sum_{i=1}^{I} \psi(x_{i+1} - x_i) = I \cdot \frac{1}{I}\sum_{i=1}^{I} \psi(x_{i+1} - x_i) \geq I \psi\prn*{\frac{x_{I+1} - x_1}{I}}
\end{equation}
Therefore, for any $x$ with $\prog{\alpha}(x) = I \leq \frac{N}{2}$,
\begin{align}
F(x) 
&= -\psi'(\zeta)x_1 + \psi(x_N) + \sum_{n=1}^{N-1}\psi(x_{i+1} - x_i) \\
&\geq -\psi'(\zeta)x_1 + I  \psi\prn*{\frac{x_{I+1} - x_1}{I}} \\
&\geq -\psi'(\zeta)x_1 + I \psi\prn*{\frac{x_1 - \alpha}{I}} \\
&\geq \inf_y -\psi'(\zeta)y + I \psi\prn*{\frac{y - \alpha}{I}}
\end{align}
The minimizing $y$ above satisfies
\begin{equation}
0 = -\psi'(\zeta) + \psi'\prn*{\frac{y - \alpha}{I}} 
\implies y = \alpha + I\zeta
\end{equation}
so
\begin{equation}
F(x) 
\geq -\psi'(\zeta)\prn*{\alpha + I\zeta} + I\psi\prn*{\zeta}
\end{equation}
Finally, we conclude that 
\begin{align}
F(x) - F^* 
&\geq I\prn*{\psi(\zeta) - \zeta\psi'(\zeta)} - \alpha\psi'(\zeta) - N\prn*{\psi(\zeta) - \zeta \psi'(\zeta)} \\
&= \prn*{N - I} \prn*{\zeta\psi'(\zeta) - \psi(\zeta)} - \alpha\psi'(\zeta) \\
&\geq \frac{N}{4\beta^2}\log\prn*{1+\frac{H\beta^2\zeta^2}{4}} - \frac{\alpha\sqrt{H}}{\beta} 
\end{align}
Where we used that 
\begin{align}
\psi(\zeta) &= \frac{\sqrt{H}\zeta}{2\beta}\arctan\prn*{\frac{\sqrt{H}\beta \zeta}{2}} - \frac{1}{2\beta^2}\log\prn*{1+\frac{H\beta^2\zeta^2}{4}} \\
\psi'(\zeta) &= \frac{\sqrt{H}}{2\beta}\arctan\prn*{\frac{\sqrt{H}\beta \zeta}{2}} \leq \frac{\pi\sqrt{H}}{4\beta} < \frac{\sqrt{H}}{\beta}
\end{align}
From here, we consider two cases, if $\beta^2 > \frac{4}{H\zeta^2}$, then
\begin{equation}
F(x) - F^* > \frac{N}{4\beta^2}\log\prn*{2} - \frac{\alpha\sqrt{H}}{\beta} > \frac{N}{6\beta^2} - \frac{\alpha\sqrt{H}}{\beta} \geq \frac{N}{12\beta^2}
\end{equation}
Otherwise, if $\beta^2 \leq \frac{4}{H\zeta^2}$ then we use that for $x \leq 1$, $\log(1+x) \geq \frac{x}{2}$ and conclude
\begin{equation}
F(x) - F^* > \frac{N}{4\beta^2}\cdot \frac{H\beta^2\zeta^2}{8} - \frac{\alpha\sqrt{H}}{\beta} = \frac{HB^2}{32N^2} - \frac{\alpha\sqrt{H}}{\beta} \geq \frac{HB^2}{64N^2}
\end{equation}
This completes the proof.
\end{proof}

Now, we define a stochastic zeroth- and first-order oracle for $F$. First, we specify the distribution over $z$:
\begin{equation}
z = \begin{cases} 
0 & \textrm{with probability } (1-p)(1-\delta) \\ 
1 & \textrm{with probability } p(1-\delta) \\
2 & \textrm{with probability } \delta
\end{cases} 
\end{equation}
Then, $f$ and $g$ are defined as
\begin{equation}\label{eq:foracle-homogeneous-convex-lower-bound}
\begin{aligned}
f(x;0) &= F(\brk*{x_1,\dots,x_{\prog{\alpha}(x)},0,\dots,0}) \\
f(x;1) &= F(\brk*{x_1,\dots,x_{\prog{\alpha}(x)+1},0,\dots,0}) \\
f(x;2) &= \frac{1}{\delta}F(x) - \frac{1-\delta}{\delta}\prn*{(1-p)f(x;0) + pf(x;1)}
\end{aligned}
\end{equation}
and
\begin{equation}\label{eq:oracle-homogeneous-convex-lower-bound}
\begin{aligned}
g(x;0) &= \sum_{i=1}^{\prog{\alpha}(x)}e_ie_i^\top \nabla F(\brk*{x_1,\dots,x_{\prog{\alpha}(x)},0,\dots,0}) \\
g(x;1) &= \frac{1}{p}\sum_{i=1}^{\prog{\alpha}(x) + 1}e_ie_i^\top\nabla F(\brk*{x_1,\dots,x_{\prog{\alpha}(x)+1},0,\dots,0}) - \frac{1-p}{p}g(x;0) \\
g(x;2) &= \frac{1}{\delta}\nabla F(x) - \frac{1-\delta}{\delta}\sum_{i=1}^{\prog{\alpha}(x) + 1}e_ie_i^\top\nabla F(\brk*{x_1,\dots,x_{\prog{\alpha}(x)+1},0,\dots,0})
\end{aligned}
\end{equation}
The following lemma relates the properties of $\psi$ to those of $F$ and $g$:
\begin{lemma}\label{lem:homogeneous-convex-generic-psi-construction}
For $F$ defined as in \eqref{eq:def-F} and the oracle $(f,g)$
\begin{enumerate}
\item $\E_z f(x;z) = F(x)$ and $\E_z g(x;z) = \nabla F(x)$
\item $\sup_x \E_z \prn*{f(x;z) - F(x)}^2 \leq \frac{12H\alpha^2}{\beta^2\delta} + \frac{3N^2H^2\alpha^4}{\delta}$
\item $\sup_x\E_z\nrm*{g(x;z) - \nabla F(x)}^2 \leq \frac{6H(1-p)}{\beta^2p} + 6NH^2\alpha^2\prn*{\frac{1}{p} + \frac{1}{\delta}}$
\item $(f,g)$ is an $(\alpha,p,\delta)$-robust-zero-chain.
\end{enumerate}
\end{lemma}
\begin{proof}
We will prove each property one by one.

\noindent
\textbf{1)} This is a simple calculation. For $f$, we have
\begin{align}
\E_z f(x;z) 
&= (1-p)(1-\delta)f(x;0) + p(1-\delta)f(x;1) + \delta f(x;2) \\
&= (1-p)(1-\delta)f(x;0) + p(1-\delta)f(x;1) \nonumber\\
&\quad+ \delta\prn*{\frac{1}{\delta}F(x) - \frac{1-\delta}{\delta}\prn*{(1-p)f(x;0) + pf(x;1)}} \\
&= F(x)
\end{align}
For the gradient, let
\begin{align}
\nabla_0 &= \sum_{i=1}^{\prog{\alpha}(x)}e_ie_i^\top \nabla F(\brk*{x_1,\dots,x_{\prog{\alpha}(x)},0,\dots,0}) \\
\nabla_1 &= \sum_{i=1}^{\prog{\alpha}(x)+1}e_ie_i^\top\nabla F(\brk*{x_1,\dots,x_{\prog{\alpha}(x)+1},0,\dots,0})
\end{align}
Then
\begin{align}
\E_z &g(x;z) \nonumber\\
&= (1-p)(1-\delta)g(x;0) + p(1-\delta)g(x;1) + \delta g(x;2) \\
&= (1-p)(1-\delta)\nabla_0 + p(1-\delta)\prn*{\frac{1}{p}\nabla_1 - \frac{1-p}{p}\nabla_0} + \delta\prn*{\frac{1}{\delta}\nabla F(x) - \frac{1-\delta}{\delta}\nabla_1} \\
&= \nabla F(x)
\end{align}  

\noindent\textbf{2)}
For any $x$
\begin{align}
&\E_z \prn*{f(x;z) - F(x)}^2\nonumber\\
&= (1-p)(1-\delta)(f(x;0) - F(x))^2 + p(1-\delta)(f(x;1) - F(x))^2 + \delta (f(x;2) - F(x))^2 \\
&= (1-p)(1-\delta)(f(x;0) - F(x))^2 + p(1-\delta)(f(x;1) - F(x))^2 \nonumber\\
&\qquad + \delta \prn*{\frac{1}{\delta}F(x) - \frac{1-\delta}{\delta}\prn*{(1-p)f(x;0) + pf(x;1)} - F(x)}^2 \\
&= (1-p)(1-\delta)(f(x;0) - F(x))^2 + p(1-\delta)(f(x;1) - F(x))^2 \nonumber\\
&\qquad + \frac{(1-\delta)^2}{\delta} \prn*{(1-p)f(x;0) + pf(x;1) - F(x)}^2 \\
&\leq \prn*{(1-p)(1-\delta) + \frac{2(1-\delta)^2(1-p)^2}{\delta}}(f(x;0) - F(x))^2 \nonumber\\
&\qquad+ \prn*{p(1-\delta) + \frac{2(1-\delta)^2p^2}{\delta}}(f(x;1) - F(x))^2 \\
&\leq \frac{3}{\delta}\prn*{(f(x;0) - F(x))^2 + (f(x;1) - F(x))^2}
\end{align}
We will address each term separately.
\begin{align}
(f(x;0) - F(x))^2
&= (F(\brk*{x_1,\dots,x_{\prog{\alpha}(x)},0,\dots,0}) - F(x))^2 \\
&= \prn*{\psi(-x_{\prog{\alpha}(x)}) - \sum_{i=\prog{\alpha}(x)}^{N-1} \psi(x_{i+1} - x_i) - \psi(x_N)}^2 \\
&\leq 2\prn*{\psi(-x_{\prog{\alpha}(x)}) - \psi(x_{\prog{\alpha}(x)+1} - x_{\prog{\alpha}(x)})}^2 \nonumber\\
&\qquad+  2\prn*{\sum_{i=\prog{\alpha}(x)+1}^{N-1} \psi(x_{i+1} - x_i) + \psi(x_N)}^2 \\
&\leq 2\prn*{\frac{\sqrt{H}}{\beta}\alpha}^2 + 2\prn*{\sum_{i=\prog{\alpha}(x)+1}^{N} \psi(2\alpha)}^2 \\
&\leq \frac{2H\alpha^2}{\beta^2} + \frac{N^2H^2\alpha^4}{2} 
\end{align}
Above, we used that $\abs{\psi'(x)} = \abs*{\frac{\sqrt{H}}{2\beta}\arctan\prn*{\frac{\sqrt{H}\beta x}{2}}} \leq \frac{\sqrt{H}}{\beta}$so $\psi$ is $\frac{\sqrt{H}}{\beta}$-Lipschitz, and also $\abs{\psi''(x)} = \frac{H}{4 + H\beta^2x^2} \leq \frac{H}{4}$ so $\psi$ is $\frac{H}{4}$-smooth. Similarly,
\begin{align}
(f(x;1) - F(x))^2
&= (F(\brk*{x_1,\dots,x_{\prog{\alpha}(x)+1},0,\dots,0}) - F(x))^2 \\
&= \prn*{\psi(-x_{\prog{\alpha}(x)+1}) - \sum_{i=\prog{\alpha}(x)+1}^{N-1} \psi(x_{i+1} - x_i) - \psi(x_N)}^2 \\
&\leq 2\prn*{\psi(-x_{\prog{\alpha}(x)+1}) - \psi(x_{\prog{\alpha}(x)+2} - x_{\prog{\alpha}(x)+1})}^2 \nonumber\\
&\qquad+  2\prn*{\sum_{i=\prog{\alpha}(x)+1}^{N-1} \psi(x_{i+1} - x_i) + \psi(x_N)}^2 \\
&\leq 2\prn*{\frac{\sqrt{H}}{\beta}\alpha}^2 + 2\prn*{\sum_{i=\prog{\alpha}(x)+1}^{N} \psi(2\alpha)}^2 \\
&\leq \frac{2H\alpha^2}{\beta^2} + \frac{N^2H^2\alpha^4}{2} 
\end{align}
Therefore,
\begin{equation}
\E_z \prn*{f(x;z) - F(x)}^2 \leq \frac{12H\alpha^2}{\beta^2\delta} + \frac{3N^2H^2\alpha^4}{\delta}
\end{equation}

\noindent
\textbf{3)} Using the same $\nabla_0$ and $\nabla_1$ as above, we first expand
\begin{align}
&\E_z \nrm*{g(x;z) - \nabla F(x)}^2 \nonumber\\
&\leq (1-p)\nrm*{\nabla_0 - \nabla F(x)}^2 + \frac{1}{p}\nrm*{\nabla_1 - (1-p)\nabla_0 - p\nabla F(x)}^2 + \frac{(1-\delta)^2}{\delta}\nrm*{\nabla_1 - \nabla F(x)}^2 \\
&\leq \frac{3(1-p)}{p}\nrm*{\nabla_0 - \nabla F(x)}^2 + \frac{3}{\delta}\nrm*{\nabla_1 - \nabla F(x)}^2 
\end{align}
Before proceeding, we recall that $\psi''(x) \leq \frac{H}{4}$ and $\abs{\psi'(x)} \leq \frac{\sqrt{H}}{\beta}$. Therefore, for $j = \prog{\alpha}(x)$,
\begin{align}
&\nrm*{\nabla_0 - \nabla F(x)}^2 \nonumber\\
&= \left\|-\psi'(\zeta)e_1 + \sum_{i=1}^{j-1}\psi'(x_{i+1} - x_i)(e_{i+1} - e_i) + \psi'(-x_j)( - e_j)\right.\nonumber\\
&\qquad\left.+ \psi'(\zeta)e_1 - \psi'(x_N)e_N - \sum_{i=1}^{N-1}\psi'(x_{i+1} - x_i)(e_{i+1} - e_i)\right\|^2 \\
&= \nrm*{-e_j\psi'(-x_j) - \psi'(x_N)e_N - \sum_{i=j}^{N-1}\psi'(x_{i+1} - x_i)(e_{i+1} - e_i)}^2 \\
&\leq 2\nrm*{-\psi'(-x_j)e_j - \psi'(x_{j+1} - x_j)(e_{j+1} - e_j)}^2 \nonumber\\
&\qquad+ 2\nrm*{\psi'(x_N)e_N + \sum_{i=j+1}^{N-1}\psi'(x_{i+1} - x_i)(e_{i+1} - e_i)}^2 \\
&\leq 2\prn*{\frac{H}{\beta^2} + \frac{H^2}{16}\alpha^2} + 32(N-j-1)\frac{H^2}{16}\alpha^2 \\
&\leq \frac{2H}{\beta^2} + 2NH^2\alpha^2
\end{align}
Similarly, 
\begin{align}
&\nrm*{\nabla_1 - \nabla F(x)}^2 \nonumber\\
&= \left\|-\psi'(\zeta)e_1 + \sum_{i=1}^{j}\psi'(x_{i+1} - x_i)(e_{i+1} - e_i) + \psi'(-x_{j+1})( - e_{j+1})\right.\nonumber\\
&\qquad\left.+ \psi'(\zeta)e_1 - \psi'(x_N)e_N - \sum_{i=1}^{N-1}\psi'(x_{i+1} - x_i)(e_{i+1} - e_i)\right\|^2 \\
&= \nrm*{-e_{j+1}\psi'(-x_{j+1}) - \psi'(x_N)e_N - \sum_{i=j+1}^{N-1}\psi'(x_{i+1} - x_i)(e_{i+1} - e_i)}^2 \\
&\leq 2\nrm*{-e_{j+1}\psi'(-x_{j+1}) - \psi'(x_{j+2} - x_{j+1})(e_{j+2} - e_{j+1})}^2 \nonumber\\
&\qquad+ 2\nrm*{\psi'(x_N)e_N + \sum_{i=j+2}^{N-1}\psi'(x_{i+1} - x_i)(e_{i+1} - e_i)}^2 \\
&\leq 10\frac{H^2}{16}\alpha^2 + 32(N-j-2)\frac{H^2}{16}\alpha^2 \\
&\leq 2NH^2\alpha^2
\end{align}
We conclude that
\begin{align}
\E_z \nrm*{g(x;z) - \nabla F(x)}^2
&\leq \frac{3(1-p)}{p}\prn*{\frac{2H}{\beta^2} + 2NH^2\alpha^2} + \frac{6NH^2\alpha^2}{\delta} \\
&\leq \frac{6H(1-p)}{\beta^2p} + 6NH^2\alpha^2\prn*{\frac{1}{p} + \frac{1}{\delta}}
\end{align}

\noindent
\textbf{4)} Comparing \eqref{eq:foracle-homogeneous-convex-lower-bound} and \eqref{eq:oracle-homogeneous-convex-lower-bound} to Definition \ref{def:robust-zero-chain-oracle}, it is easy to see that $(f,g)$ is an $(\alpha,p,\delta)$-robust-zero-chain with $\mc{Z}_0 = \crl{0}$ and $\mc{Z}_1 = \crl{1}$.
\end{proof}

\begin{lemma}\label{lem:statistical-term-lower-bound}
Let $H,B,\sigma^2 > 0$ and let $g$ be a stochastic gradient oracle with variance bounded by $\sigma^2$. Then for any algorithm that accesses the oracle $T$ times, there exists a function in one dimension such that the algorithm's output will have error at least 
\[
\E F(\hat{x}) - F^* \geq \frac{3}{8}\min\crl*{\frac{\sigma B}{\sqrt{T}}, HB^2}
\]
\end{lemma}
\begin{proof}
Consider the following pair of objectives:
\begin{equation}
\begin{aligned}
F_+(x) &= \frac{a}{2}x^2 + bx \\
F_-(x) &= \frac{a}{2}x^2 - bx
\end{aligned}
\end{equation}
with a stochatic gradient oracles
\begin{equation}
\begin{aligned}
z &\sim \mc{N}(0,\sigma^2) \\
g_+(x;z) &= \nabla F_+(x) + z \\
g_-(x;z) &= \nabla F_-(x) + z
\end{aligned}
\end{equation}
First, we note that for any $x$, 
\begin{equation}
F_+(x) - \min_x F_+(x) \leq \frac{b^2}{2a} \implies x \leq 0 \implies F_-(x) - \min_x F_-(x) \geq \frac{b^2}{2a}
\end{equation}
and vice versa. Therefore, any algorithm that succeeds in optimizing both $F_+$ and $F_-$ to accuracy better than $\frac{b^2}{2a}$ with probability at least $\frac{3}{4}$ needs to determine which of the two functions it is optimizing with probability at least $\frac{3}{4}$. However, by the Pinsker inequality, the total variation distance between $T$ queries to $g_+$ and $g_-$ is at most
\begin{align}
\nrm{\P_+ - \P_-}_{\textrm{TV}} 
&\leq \sqrt{\frac{1}{2}\textrm{D}_{\textrm{KL}}\prn*{\P_+\middle\|\P_-}} \\
&\leq \sqrt{\frac{T}{2}\textrm{D}_{\textrm{KL}}\prn*{\mc{N}(2b,\sigma^2)\middle\|\mc{N}(0,\sigma^2)}} \\
&= \frac{b}{\sigma}\sqrt{T}
\end{align}
Therefore, if $b \leq \frac{3\sigma}{4\sqrt{T}}$, no algorithm can optimize to accuracy better $\frac{b^2}{2a}$ with probability greater than $\frac{3}{4}$. 
Finally, we note that $F_+$ and $F_-$ are $a$-smooth, and have minimizers $\mp \frac{b}{a}$. Therefore, we take $b = \min\crl*{aB, \frac{3\sigma}{4\sqrt{T}}}$ and $a = \min\crl*{H,\ \frac{3\sigma}{4B\sqrt{T}}}$ so that the objectives are $H$-smooth and have solutions of norm $B$ and with probability at least $\frac{1}{4}$
\begin{equation}
\max_{* \in \crl{+,-}} \E F_*(\hat{x}) - \min_x F_*(x) 
\geq \min\crl*{\frac{aB^2}{2}, \frac{9\sigma^2}{32aT}} 
\geq \min\crl*{\frac{HB^2}{2}, \frac{3\sigma B}{8\sqrt{T}}}
\end{equation}
This completes the proof.
\end{proof}

\mainlowerbound*
\begin{proof}
We set the parameters of the objective \eqref{eq:def-F} and the oracles \eqref{eq:foracle-homogeneous-convex-lower-bound} and \eqref{eq:oracle-homogeneous-convex-lower-bound} according to 
\begin{align}
p &= \min\crl*{1,\,\frac{\sqrt{HB}}{\sqrt{\sigma}K^{3/4}R^{3/4}}} \\
\delta &= \frac{1}{16MKR} \\
N &= \min\crl*{2KR, 16KRp + 24R (1 + \log M)} \\
\alpha &= \min\crl*{\frac{N}{12\beta\sqrt{H}},\,\frac{\sqrt{H}\beta B^2}{64N^2},\,\frac{\sigma\sqrt{p\delta}}{5H\sqrt{N}}} \\
\beta &= \frac{2N^{3/2}}{\sqrt{H}B} 
\end{align}

By Lemma \ref{lem:F-properties}, the objective \eqref{eq:def-F} is convex, $H$-smooth, and has a minimizer with norm less than $B$. In order to apply Lemma \ref{lem:intermittent-communication-progress}, we will also introduce a random rotation $U$, but since $U^\top U = I_{N\times N}$, this does not affect the convexity, smoothness, or norm of the minimizers of the objective, so we conclude that our construction satisfies the necessary conditions. 

Furthermore, by Lemma \ref{lem:homogeneous-convex-generic-psi-construction}, the stochastic zeroth- and first-order oracles defined in \eqref{eq:foracle-homogeneous-convex-lower-bound} and \eqref{eq:oracle-homogeneous-convex-lower-bound} are unbiased and, with our choice of $\alpha \leq \frac{\sigma\sqrt{p\delta}}{5H\sqrt{N}}$ and (as we will show) $p \geq \frac{12H}{12H + \beta^2\sigma^2}$, they have variance bounded by 
\begin{align}
\sup_x \E_z \prn*{f(x;z) - F(x)}^2 
&\leq \frac{12H\alpha^2}{\beta^2\delta} + \frac{3N^2H^2\alpha^4}{\delta} \\
&\leq \frac{12\sigma^2p}{25HN\beta^2} + \frac{3\sigma^4p^2\delta}{625H^2} =: \rho^2 \\
\sup_x\E_z\nrm*{g(x;z) - \nabla F(x)}^2 
&\leq \frac{6H(1-p)}{\beta^2p} + 6NH^2\alpha^2\prn*{\frac{1}{p} + \frac{1}{\delta}} 
\leq \sigma^2
\end{align}

Next, on the way to applying Lemma \ref{lem:intermittent-communication-progress}, we introduce a uniformly random rotation for $U \in \R^{d\times N}$, and consider $F(U^\top x)$ with oracle $(f_U,g_U)$. With our choice of $d$, we note that
\begin{equation}
d \geq N + \frac{2\gamma^2}{\alpha^2}\log(32MKRN)
\end{equation}
for
\begin{equation}
\gamma = 2B + \max\crl*{408B,\,\sqrt{\frac{32\rho^2}{\sigma^2}}}
\end{equation}
Therefore, by Lemma \ref{lem:intermittent-communication-progress}, for any algorithm whose queries are bounded in norm by $\gamma$, the algorithm's output $\hat{x}$ will satisfy 
\begin{equation}
\P\prn*{\prog{\alpha}(U^\top \hat{x}) \leq \min\crl*{KR, 8KRp + 12R (1 + \log M)}} \geq \frac{5}{8} - 2MKR\delta \geq \frac{1}{2}
\end{equation}
Therefore, with our choice of 
\begin{equation}
N = \min\crl*{2KR, 16KRp + 24R (1 + \log M)}
\end{equation}
we have that $\prog{\alpha}(U^\top \hat{x}) \leq \frac{N}{2}$ with probability at least $\frac{1}{2}$. It therefore follows from Lemma \ref{lem:F-properties} that with probability at least $\frac{1}{2}$
\begin{align}
F(U^\top \hat{x}) - F^* 
&\geq \begin{cases}
\frac{N}{12\beta^2} & \beta^2 > \frac{4N^3}{HB^2} \\
\frac{HB^2}{64N^2} & \beta^2 \leq \frac{4N^3}{HB^2} 
\end{cases} \\
&= \frac{HB^2}{64\min\crl*{2KR, 16KRp + 24R (1 + \log M)}^2} \\
&\geq \frac{HB^2}{512K^2R^2} + \frac{HB^2}{32768K^2R^2p^2 + 73728R^2 (1 + \log M)^2} \\
&\geq \frac{1}{73728}\prn*{\frac{HB^2}{K^2R^2} + \min\crl*{\frac{HB^2}{K^2R^2p^2}, \frac{HB^2}{R^2 (1 + \log M)^2}}}
\end{align}
From here, what remains is to show how small $p$ can be taken. Above, we claimed that our choice satisfies $p \geq \frac{12H}{12H + \beta^2\sigma^2}$, but this is a more complicated statement than it appears since $\beta$ is defined in terms of $N$, which is, in turn, defined in terms of $p$. We have
\begin{align}
\frac{12H}{12H + \beta^2\sigma^2}
&= \frac{12H}{12H + \frac{4\sigma^2N^3}{HB^2}} \\
&= \frac{4H^2B^2}{4H^2B^2 + 4\sigma^2N^3} \\
&= \frac{4H^2B^2}{4H^2B^2 + 4\sigma^2\min\crl*{2KR, 16KRp + 24R (1 + \log M)}^3} \\
&\leq \frac{4H^2B^2}{4H^2B^2 + 4\sigma^2(2KRp)^3}
\end{align}
Therefore, it suffices to set
\begin{align}
p \geq \frac{4H^2B^2}{4H^2B^2 + 4\sigma^2(2KRp)^3} \\
\iff 32\sigma^2K^3R^3p^4 + 4H^2B^2p \geq 4H^2B^2 \\
\impliedby p^4 \geq \min\crl*{1,\,\frac{H^2B^2}{8\sigma^2K^3R^3}}
\end{align}
so, our choice of $p$ is sound. Therefore, we can lower bound
\begin{align}
\E F(U^\top \hat{x}) - F^* 
&\geq \frac{1}{2}\cdot \frac{1}{73728}\prn*{\frac{HB^2}{K^2R^2} + \min\crl*{\frac{HB^2}{K^2R^2p^2}, \frac{HB^2}{R^2 (1 + \log M)^2}}} \\
&\geq \frac{1}{2}\cdot\frac{1}{73728}\prn*{\frac{HB^2}{K^2R^2} + \min\crl*{\frac{HB^2}{K^2R^2\frac{HB}{\sigma K^{3/2}R^{3/2}}}, \frac{HB^2}{R^2 (1 + \log M)^2}}} \\
&\geq \frac{1}{2}\cdot\frac{1}{73728}\prn*{\frac{HB^2}{K^2R^2} + \min\crl*{\frac{\sigma B}{\sqrt{KR}}, \frac{HB^2}{R^2 (1 + \log M)^2}}}
\end{align}
This lower bound applies to all algorithm's whose queries are bounded by $\gamma$. To extend the result to all randomized algorithms, we apply Lemma \ref{lem:large-norm}, which results in only a constant factor degredation in the lower bound.
Finally, we apply Lemma \ref{lem:statistical-term-lower-bound} to conclude that 
\begin{equation}
\E F(\hat{x}) - F^* \geq \frac{1}{4}\min\crl*{HB^2,\, \frac{\sigma B}{\sqrt{MKR}}}
\end{equation}
This completes the proof.
\end{proof}

\sclowerbound*
\begin{proof}
Suppose there were an algorithm which guaranteed convergence at a rate 
\begin{multline}\label{eq:speculative}
G(\hat{x}) - G^* \leq c\cdot\bigg(\frac{G(0) - G^*}{K^2R^2}\exp\prn*{-\sqrt{\frac{\lambda}{H}}KR} + \frac{\sigma^2}{\lambda MKR} \\
+ \min\crl*{\frac{G(0) - G^*}{R^2\log^2 M}\exp\prn*{-\sqrt{\frac{\lambda}{H}}R\log M},\, \frac{\sigma^2}{\lambda KR}}\bigg)
\end{multline}
for any $\lambda$-strongly convex function $G$. Then, we could use this algorithm to optimize a merely convex $F$ with $\nrm{x^*}\leq B$ by applying it to the $\lambda$-strongly convex $G(x) = F(x) + \frac{\lambda}{2}\nrm{x}^2$. Using $x^*_G$ to denote the minimizer of $G$, this would ensure (for a universal constant $c$ which may change from line to line)
\begin{align}
F(\hat{x}) - F^*
&\leq G(\hat{x}) - F^* \\
&= G(\hat{x}) - G(x^*) + \frac{\lambda}{2}\nrm*{x^*}^2 \\
&\leq G(\hat{x}) - G(x_G^*) + \frac{\lambda B^2}{2} \\
&\leq c\cdot\bigg(\frac{G(0) - G^*}{K^2R^2}\exp\prn*{-\sqrt{\frac{\lambda}{H}}KR} + \frac{\sigma^2}{\lambda MKR}  \\
&\qquad+ \min\crl*{\frac{G(0) - G^*}{R^2\log^2 M}\exp\prn*{-\sqrt{\frac{\lambda}{H}}R\log M},\, \frac{\sigma^2}{\lambda KR}}\bigg) + \lambda B^2 \\
&\leq c\cdot\bigg(\frac{HB^2}{K^2R^2}\exp\prn*{-\sqrt{\frac{\lambda}{H}}KR} + \frac{\sigma^2}{\lambda MKR}  \\
&\qquad+ \min\crl*{\frac{HB^2}{R^2\log^2 M}\exp\prn*{-\sqrt{\frac{\lambda}{H}}R\log M},\, \frac{\sigma^2}{\lambda KR}} + \lambda B^2\bigg)
\end{align}
Consequently, if we choose
\begin{equation}
\lambda = \max\crl*{\frac{H}{K^2R^2},\, \frac{\sigma}{B\sqrt{MKR}},\, \frac{H}{R^2\log^2M},\, \frac{\sigma}{B\sqrt{KR}}}
\end{equation}
then this approach would guarantee 
\begin{align}
F(\hat{x}) - F^*
&\leq c\cdot\bigg(\frac{HB^2}{K^2R^2} + \frac{\sigma B}{\sqrt{MKR}} 
+ \min\crl*{\frac{HB^2}{R^2\log^2 M},\, \frac{\sigma B}{\sqrt{KR}}} \\
&\quad+ B^2\max\crl*{\frac{H}{K^2R^2},\, \frac{\sigma}{B\sqrt{MKR}},\, \frac{H}{R^2\log^2M},\, \frac{\sigma}{B\sqrt{KR}}} \bigg) \\
&= c\cdot\prn*{\frac{HB^2}{K^2R^2} + \frac{\sigma B}{\sqrt{MKR}} 
+ \min\crl*{\frac{HB^2}{R^2\log^2 M},\, \frac{\sigma B}{\sqrt{KR}}}}
\end{align}
In light of the lower bound in Theorem \ref{thm:lower-bound}, we conclude that no algorithm can provide a guarantee that is more than a constant factor better than \eqref{eq:speculative}.
\end{proof}

\section{Proof of Theorem \ref{thm:bounded-third}}\label{app:thm5-proofs}
In this section, we extend Theorem \ref{thm:lower-bound} to the case where the objective is required to exhibit higher-order smoothness. Although we are confident that similar results as Theorem \ref{thm:bounded-third} would apply to arbitrary randomized algorithms, we prove the lower bound here just for zero-respecting algorithms \citep{carmon2017lower1}:
\begin{definition}[Distributed Zero-Respecting Intermittent Communication Algorithm]\label{def:zr}
We say that a parallel method is an intermittent communication algorithm if for each $m,k,r$, there exists a mapping $\mc{A}^m_{k,r}$ such that $x^m_{k,r}$, the $k^{\textrm{th}}$ query on the $m^{\textrm{th}}$ machine during the $r^{\textrm{th}}$ round of communication, is computed as
\[
x^m_{k,r} = \mc{A}^m_{k,r}\prn*{\brk*{x^{m'}_{k',r'}, g\prn*{x^{m'}_{k',r'}}}_{m'\in[M],k'\in[K],r'< r}, \brk*{x^{m\phantom{'}}_{k',r}, g\prn*{x^{m\phantom{'}}_{k',r}}}_{k' < k}, \xi}
\]
where $\xi$ is a string of random bits that the algorithm may use for randomization. In addition, for a vector $x$, we define $\textrm{support}(x) := \crl*{j\,:\, x_j \neq 0}$, and we say that an intermittent communication algorithm is distributed zero-respecting
\[
\textrm{support}(x^m_{k,r}) \subseteq \bigcup_{m'\in[M],k'\in[K],r'<r} \textrm{support}\prn*{g(x^{m'}_{k',r'})} \cup \bigcup_{k'<k} \textrm{support}\prn*{g(x^{m\phantom{'}}_{k',r})}
\]
\end{definition}

We construct a hard instance for the lower bound using the scalar functions $\psi:\R\to\R$:
\begin{equation}
\psi(x) = \frac{\sqrt{H}x}{2\beta}\arctan\prn*{\frac{\sqrt{H}\beta x}{2}} - \frac{1}{2\beta^2}\log\prn*{1+\frac{H\beta^2x^2}{4}}
\end{equation}
where $H$ is the parameter of smoothness, and $\beta > 0$ is another parameter that controls the third derivative of $\psi$ which we will set later. The hard instance is then
\begin{equation}\label{eq:def-F5}
F(x) = -\psi'(\zeta) x_1 + \psi(x_N) + \sum_{i=1}^{N-1} \psi(x_{i+1} - x_i)
\end{equation}
where $\zeta$ and $N$ are additional parameters that will be chosen later. Lemma \ref{lem:F-properties5} below summarizes the relevant properties of $F$, whose proof relies on the following bounds on $\psi'''$:
\begin{lemma}\label{lem:technical-1}
For any $H,\beta \geq 0$,
\begin{align*}
\abs{\psi'''(x)} &\leq \frac{H^{3/2}\beta}{12} \\
\abs{\psi'''(x)} &\leq 2\beta\psi''(x)^{3/2} \\
\abs{\psi'''(x)} &\leq \frac{\sqrt{H}\beta}{2}\psi''(x)
\end{align*}
\end{lemma}
\begin{proof}
The third derivative of $\psi$ is
\begin{equation}
\psi'''(x) = \frac{-2H^2\beta^2 x}{(4 + H\beta^2 x^2)^2}
\end{equation}
For the first claim, we first maximize the simpler function $x\mapsto \frac{x}{(1 + x^2)^2}$. We note that
\begin{gather}
\frac{d}{dx} \frac{x}{(1 + x^2)^2} = \frac{1 - 3x^2}{(1+x^2)^3} \\
\frac{d^2}{dx^2} \frac{x}{(1 + x^2)^2} = \frac{12x(x^2-1)}{(1+x^2)^4}
\end{gather}
Therefore, the derivative is zero at $\pm 1/\sqrt{3}$ and the second derivative is negative only for $+ 1/\sqrt{3}$, furthermore, $\lim_{x\to\pm\infty} \frac{x}{(1 + x^2)^2} = 0$. Therefore, we conclude that
\begin{equation}
\max_{x\in\R} \frac{x}{(1 + x^2)^2} = \max_{x\in\R} \frac{\abs{x}}{(1 + x^2)^2}
= \frac{\sqrt{\frac{1}{3}}}{(1 + \sqrt{\frac{1}{3}}^2)^2} 
= \frac{3\sqrt{3}}{16}
\end{equation}
By rescaling, we conclude that
\begin{equation}
\max_{x\in\R} \abs*{\psi'''(x)}
= \max_{x\in\R} \frac{2H^2\beta^2\abs{x}}{(4 + H\beta^2x^2)^2}
= \frac{H^{3/2}\beta}{4}\max_{x\in\R} \frac{\abs*{\frac{\sqrt{H}\beta x}{2}}}{\prn*{1 + \prn*{\frac{\sqrt{H}\beta x}{2}}^2}^2}
= \frac{3\sqrt{3}H^{3/2}\beta}{64}
< \frac{H^{3/2}\beta}{12}
\end{equation}
This establishes the first claim. For the second claim, we observe that
\begin{equation}
\abs*{\psi'''(x)} 
= \frac{2\sqrt{H}\beta^2\abs{x}}{\sqrt{4 + H\beta^2 x^2}}\psi''(x)^{3/2}  
\leq \frac{2\sqrt{H}\beta^2\abs{x}}{\sqrt{H\beta^2 x^2}}\psi''(x)^{3/2}
= 2\beta\psi''(x)^{3/2}
\end{equation}
Finally, for the third claim, we start by noting
\begin{equation}
\abs*{\psi'''(x)} 
= \frac{2H\beta^2\abs{x}}{4 + H\beta^2 x^2}\psi''(x)  
\end{equation}
We now consider the function $x \mapsto \frac{x}{1+x^2}$, for which
\begin{align}
\frac{d}{dx}\frac{x}{1+x^2} &= \frac{1-x^2}{(1+x^2)^2} \\
\frac{d^2}{dx^2}\frac{x}{1+x^2} &= \frac{2x(x^2-3)}{(1+x^2)^3}
\end{align}
We conclude that
\begin{equation}
\max_{x\in\R} \frac{\abs{x}}{1+x^2} = \frac{1}{1 + 1^2} = \frac{1}{2}
\end{equation}
and therefore,
\begin{equation}
\max_{x\in\R} \frac{2H\beta^2\abs{x}}{4 + H\beta^2 x^2} 
= \sqrt{H}\beta \max_{x\in\R} \frac{\abs*{\frac{\sqrt{H}\beta x}{2}}}{1 + \prn*{\frac{\sqrt{H}\beta x}{2}}^2} = \frac{\sqrt{H}\beta}{2}
\end{equation}
This completes the proof.
\end{proof}

\begin{lemma}\label{lem:F-properties5}
For any $H \geq 0$, $\beta > 0$, $\zeta > 0$, and $N \geq 2$, $F$ is convex, $H$-smooth, $\beta$-self-concordant, $\frac{\sqrt{H}\beta}{2}$-quasi-self-concordant, and $\nrm{\nabla^3 F(x)} \leq \frac{4H^{3/2}\beta}{3}$.
\end{lemma}
\begin{proof}
First, we note that $0 \leq \psi''(x) = \frac{H}{4 + H\beta^2 x^2} \leq \frac{H}{4}$. Therefore, $F$ is the sum of convex functions and is thus convex itself. We now compute the Hessian of $F$:
\begin{equation}
\nabla^2 F(x) = \psi''(x_N)e_Ne_N^\top + \sum_{i=1}^{N-1}\psi''(x_{i+1} - x_i)(e_{i+1}-e_i)(e_{i+1}-e_i)^\top
\end{equation}
Therefore, for any $u \in \R$,
\begin{align}
u^\top \nabla^2 F(x) u 
&\leq \psi''(x_N) u_N^2 + \sum_{i=1}^{N-1}\psi''(x_{i+1} - x_i)(u_{i+1}-u_i)^2 \\
&\leq \frac{H}{4}\brk*{u_N^2 + \sum_{i=1}^{N-1}2u_{i+1}^2 + 2u_i^2} \\
&\leq H\nrm{u}^2
\end{align}
We conclude that $\nabla^2 F(x) \preceq H\cdot I$ and thus $F$ is $H$-smooth.

Next, we compute the tensor of 3rd derivatives of $F$:
\begin{equation}
\nabla^3 F(x) = \psi'''(x_N)e_N^{\otimes 3} + \sum_{i=1}^{N-1}\psi'''(x_{i+1} - x_i)(e_{i+1}-e_i)^{\otimes 3}
\end{equation}
where
\begin{equation}
\psi'''(x) = \frac{-2H^2\beta^2 x}{(4 + H\beta^2 x^2)^2}
\end{equation}
Therefore, for any $u\in\R$,
\begin{equation}
\abs*{\nabla^3 F(x)[u,u,u]} \leq \abs*{\psi'''(x_N)u_N^3} + \sum_{i=1}^{N-1}\abs*{\psi'''(x_{i+1} - x_i)(u_{i+1}-u_i)^3}
\end{equation}
We can bound this in several different ways using Lemma \ref{lem:technical-1}:
\begin{align}
\abs{\psi'''(x)} &\leq \frac{H^{3/2}\beta}{12} \\
\abs{\psi'''(x)} &\leq 2\beta\psi''(x)^{3/2} \\
\abs{\psi'''(x)} &\leq \frac{\sqrt{H}\beta}{2}\psi''(x)
\end{align}
Therefore,
\begin{align}
\abs*{\nabla^3 F(x)[u,u,u]} 
&\leq \abs*{\psi'''(x_N)u_N^3} + \sum_{i=1}^{N-1}\abs*{\psi'''(x_{i+1} - x_i)(u_{i+1}-u_i)^3} \\
&\leq \frac{H^{3/2}\beta}{12}\brk*{\abs*{u_N}^3 + 8\sum_{i=1}^{N-1}\abs*{u_{i+1}}^3 + \abs{u_i}^3} \\
&\leq \frac{4H^{3/2}\beta}{3}\nrm*{u}^3
\end{align}
Above, we used that $\abs{a-b}^3 \leq (\abs{a}+\abs{b})^3 \leq 8(\abs{a}^3+\abs{b}^3)$. 
We conclude that $\nrm{\nabla^3 F(x)} \leq \frac{4H^{3/2}\beta}{3}$.

Similarly,
\begin{align}
\abs{\nabla^3 F(x)[u,u,u]} 
&\leq \abs{\psi'''(x_N)}\abs{u_N}^3 + \sum_{i=1}^{N-1}\abs{\psi'''(x_{i+1} - x_i)}\abs{u_{i+1}-u_i}^3 \\
&\leq 2\beta\brk*{\psi''(x_N)^{3/2} \prn{u_N^2}^{3/2} + \sum_{i=1}^{N-1}\psi''(x_{i+1} - x_i)^{3/2}\prn{(u_{i+1}-u_i)^2}^{3/2}} \\
&\leq 2\beta\brk*{\psi''(x_N) u_N^2 + \sum_{i=1}^{N-1}\psi''(x_{i+1} - x_i)(u_{i+1}-u_i)^2}^{3/2} \\
&= 2\beta \inner{\nabla^2 F(x)u}{u}^{3/2}
\end{align}
For the final inequality, we used that $\abs{a}^{3/2} + \abs{b}^{3/2} \leq (\abs{a}+\abs{b})^{3/2}$. We conclude that $F$ is $\beta$-self-concordant. 

Finally,
\begin{align}
\abs{\nabla^3 F(x)[u,u,u]} 
&\leq \abs{\psi'''(x_N)}\abs{u_N}^3 + \sum_{i=1}^{N-1}\abs{\psi'''(x_{i+1} - x_i)}\abs{u_{i+1}-u_i}^3 \\
&\leq \frac{\sqrt{H}\beta}{2}\brk*{\psi''(x_N) \abs{u_N}^3 + \sum_{i=1}^{N-1}\psi''(x_{i+1} - x_i)\abs{u_{i+1}-u_i}^3} \\
&\leq \frac{\sqrt{H}\beta}{2}\brk*{\psi''(x_N) \abs{u_N}^2 + \sum_{i=1}^{N-1}\psi''(x_{i+1} - x_i)\abs{u_{i+1}-u_i}^2} \\
&\qquad\qquad\qquad\qquad\qquad\qquad\qquad\qquad\cdot\max\crl*{\abs{u_N},\,\max_{1\leq i \leq N-1} \abs{u_{i+1}-u_i}} \nonumber\\
&\leq \frac{\sqrt{H}\beta}{2}\nrm{u}\nabla^2 F(x)[u,u]
\end{align}
For the second to last line, we applied the H\"older inequality $\sum_i \abs{a_ib_i} \leq \nrm{a}_1 \nrm{b}_\infty$. We conclude that $F$ is $\frac{\sqrt{H}\beta}{2}$-quasi-self-concordant.
\end{proof}

We will now proceed to construct a stochastic gradient oracle for $F$. To do so, we define $\prg(x)$ to be the highest index of a non-zero coordinate of $x$:
\begin{equation}
    \prg(x) = \prog{0}(x) = \max\crl*{j\,:\, x_j \neq 0}
\end{equation}
With this in hand, we define $F^-$ to be equal to the objective with the $\prg(x)^{\textrm{th}}$ term removed:
\begin{equation}
F^-(x) = \psi'(-\zeta) x_1 + \psi(x_N) + \sum_{i=1}^{\prg(x)-1} \psi(x_{i+1} - x_i) + \sum_{i=\prg(x)+1}^{N-1} \psi(x_{i+1} - x_i)
\end{equation}
The stochastic gradient oracle for $F$ is then given by 
\begin{equation}\label{eq:def-g}
g(x) = \begin{cases}
\nabla F^-(x) & \textrm{with probability } 1-p \\
\nabla F(x) + \frac{1-p}{p}\prn*{\nabla F(x) - \nabla F^-(x)} & \textrm{with probability } p 
\end{cases}
\end{equation}
The following Lemma shows that $g$ is a suitable stochastic gradient oracle for $F$:
\begin{lemma}\label{lem:stochastic-gradient-oracle}
For any $H,\beta,\sigma,\zeta,N$, if $p \geq \frac{\pi^2 H}{\pi^2 H + 8\sigma^2\beta^2}$ then for any $x$
\begin{align*}
\E g(x) &= \nabla F(x) \\
\E\nrm*{g(x) - \nabla F(x)}^2 &\leq \sigma^2
\end{align*}
\end{lemma}
\begin{proof}
First, we compute the expectation of $g(x)$:
\begin{equation}
\E g(x) = (1-p) \nabla F^-(x) + p\prn*{\nabla F(x) + \frac{1-p}{p}\prn*{\nabla F(x) - \nabla F^-(x)}} = \nabla F(x)
\end{equation}
Second, the variance can be bounded by
\begin{align}
\E\nrm*{g(x) - \nabla F(x)}^2
&= (1-p)\nrm*{\nabla F^-(x) - \nabla F(x)}^2 + p\nrm*{\frac{1-p}{p}\prn*{\nabla F(x) - \nabla F^-(x)}}^2 \\
&= \frac{1-p}{p}\nrm*{\psi'(x_{\prg(x)+1}-x_{\prg(x)})(e_{\prg(x)+1}-e_{\prg(x)})}^2 \\
&\leq \frac{2(1-p)}{p}\sup_{x\in\R} \prn*{\psi'(x)}^2 \\
&= \frac{2(1-p)}{p}\cdot \frac{\pi^2 H}{16\beta^2}
\end{align}
Therefore, taking $p \geq \frac{\pi^2 H}{\pi^2 H + 8\sigma^2\beta^2}$ ensures the variance is bounded by $\sigma^2$.
\end{proof}

In order to prove the lower bound, we will show that with constant probability, all of the iterates generated by any distributed zero-respecting intermittent communication algorithm will have progress $\prg(x) \leq N/2$, and we will proceed to show that this implies high suboptimality. The next Lemma upper bounds the progress of the algorithm's iterates:
\begin{lemma}\label{lem:progress}
For any $H,\beta,\zeta,\sigma,K,R > 0$ and $N,M \geq 2$, let $p = \max\crl*{\frac{2}{K},\frac{\pi^2 H}{\pi^2 H + 8\sigma^2 \beta^2}}$. Then with probability at least $\frac{1}{2}$, all of the oracle queries made by any distributed zero-respecting intermittent communication algorithm will have progress at most
\[
\max_{m,k,r}\prg(x^m_{k,r}) \leq \min\crl*{RK,\,\max\crl*{48R\log M,\, \frac{4\pi^2 H KR}{\pi^2 H + 8\sigma^2\beta^2}}}
\]
\end{lemma}
\begin{proof}
To begin, fix any vector $x$ and let $j = \prg(x)$. Then since $\psi'(0) = 0$, 
\begin{align}
\nabla F^-(x) 
&= \psi'(-\zeta)e_1 + \psi'(x_N)e_N + \sum_{i\neq j} \psi'(x_{i+1}-x_i)(e_{i+1}-e_i) \\
&= \psi'(-\zeta)e_1 + \sum_{i=1}^{j-1} \psi'(x_{i+1}-x_i)(e_{i+1}-e_i) \\
&\in \textrm{span}\crl*{e_1,\dots,e_j}
\end{align}
Therefore, $\prg(\nabla F^-(x)) \leq \prg(x)$, so 
\begin{equation}
\P\brk*{\prg(g(x)) > \prg(x)} = p
\end{equation}
By a similar argument, is also easy to confirm that $\prg(g(x)) \leq \prg(x) + 1$. 

By the definition of a distributed zero-respecting algorithm, the $k^{\textrm{th}}$ oracle query on the $m^{\textrm{th}}$ machine in the $r^{\textrm{th}}$ round of communication has progress no greater than the highest progress of any stochastic gradient that is available, i.e.~any stochastic gradients computed by any machine in rounds $1,\dots,r-1$, and the first $k-1$ gradients computed on machine $m$ in round $r$. As shown above, each stochastic gradient oracle query allows the algorithm to increase its progress by at most one, and only with probability $p$. 

Therefore, the maximum amount of progress that can be made on the $m^{\textrm{th}}$ machine during the $r^{\textrm{th}}$ round of communication is upper bounded by a $\textrm{Binomial}(K,p)$ random variable, and the total progress made by all the machines during the $r^{\textrm{th}}$ round is upper bounded by the maximum of $M$ independent $\textrm{Binomial}(K,p)$ random variables. Let $n_r^m \sim \textrm{Binomial}(K,p)$ denote the amount of progress made by the $m^{\textrm{th}}$ machine during the $r^{\textrm{th}}$, then for any $n$
\begin{equation}\label{eq:progress-eq1}
\P\brk*{\max_{m,k,r} \prg(x^m_{k,r}) > n}
\leq \P\brk*{\sum_{r=1}^R \max_{1\leq m \leq M} n_r^m > n} 
\end{equation}
To start, by the union bound and then the Chernoff bound, for each $r$ and any $\epsilon > 0$
\begin{equation}
\P\brk*{\max_{1\leq m \leq M} n_r^m \geq (1+\epsilon)Kp} 
\leq M\P\brk*{n_r^1 \geq (1+\epsilon)Kp} 
\leq M\exp\prn*{-\frac{\epsilon^2 Kp}{2 + \epsilon}}
\end{equation}
For any random variable $X \in [0,K]$, $\E X = \int_0^K \P\brk*{X \geq x} dx$. Therefore, for each $r$ and any $\epsilon > 0$
\begin{align}
\E\brk*{\max_{1\leq m \leq M} n_r^m}
&= \int_0^{(1+\epsilon)Kp}\P\brk*{\max_{1\leq m \leq M} n_r^m \geq x} dx + \int_{(1+\epsilon)Kp}^{K}\P\brk*{\max_{1\leq m \leq M} n_r^m \geq x} dx \\
&\leq (1+\epsilon)Kp + \int_{\epsilon}^{\frac{1-p}{p}}\P\brk*{\max_{1\leq m \leq M} n_r^m \geq (1+c)Kp} dc \\
&\leq (1+\epsilon)Kp + M\int_{\epsilon}^{\infty}\exp\prn*{-\frac{c^2 Kp}{2 + c}} dc \\
&\leq (1+\epsilon)Kp + M\int_{\epsilon}^{\infty}\exp\prn*{-\frac{c \epsilon Kp}{2 + \epsilon}} dc \\
&= (1+\epsilon)Kp + \frac{M(2+\epsilon)}{\epsilon Kp}\exp\prn*{-\frac{\epsilon^2Kp}{2+\epsilon}}
\end{align}
We apply this result with $\epsilon = 2 + \frac{2\log M}{Kp}$ so, recalling that $p \geq 2/K$ and $M \geq 2$,
\begin{align}
\E\brk*{\max_{1\leq m \leq M} n_r^m}
&\leq (1+\epsilon)Kp + \frac{M(2+\epsilon)}{\epsilon Kp} \exp\prn*{-\frac{\epsilon^2Kp}{2+\epsilon}} \\
&\leq 4Kp + 2\log M \\
&= \max\crl*{8,\, \frac{\pi^2 H K}{\pi^2 H + 8\sigma^2\beta^2}} + 2\log M \\
&\leq \max\crl*{24\log M,\, \frac{2\pi^2 H K}{\pi^2 H + 8\sigma^2\beta^2}}
\end{align}
Therefore, in light of \eqref{eq:progress-eq1} we use Markov's inequality to conclude
\begin{align}
\P\brk*{\max_{m,k,r} \prg(x^m_{k,r}) > 2R\E\brk*{\max_{1\leq m \leq M} n_r^m}}
\leq \P\brk*{\sum_{r=1}^R \max_{1\leq m \leq M} n_r^m > 2R\E\brk*{\max_{1\leq m \leq M} n_r^m}} \leq \frac{1}{2}
\end{align}
We conclude by substituting for $\E\brk*{\max_{1\leq m \leq M} n_r^m}$ and noting that $n^m_r \leq K$ always.
\end{proof}

The final piece of the proof is to show that if $\prg(x) \leq N/2$ then $F(x) - F^*$ is large:
\begin{lemma}\label{lem:suboptimality}
For any $H,\beta,B > 0$ and $N \geq 2$, set $\zeta^2 = \frac{B^2}{N^3}$. Then, $\nrm{x^*}\leq B$ and for any $x$ such that $\prg(x) \leq \frac{N}{2}$, 
\[
F(x) - F^* \geq \begin{cases}
\frac{N}{6\beta^2} & \beta^2 > \frac{4N^3}{HB^2} \\
\frac{HB^2}{48N^2} & \beta^2 \leq \frac{4N^3}{HB^2}
\end{cases}
\]
\end{lemma}
\begin{proof}
The first-order optimality condition $\nabla F(x^*) = 0$ indicates
\begin{equation}
\begin{aligned}
[\nabla F(x^*)]_1 = 0 &= \psi'(-\zeta) - \psi'(x^*_2 - x^*_1) \\
[\nabla F(x^*)]_i = 0 &= \psi'(x^*_i - x^*_{i-1}) - \psi'(x^*_{i+1} - x^*_i) \qquad 1 < i < N \\
[\nabla F(x^*)]_N = 0 &= \psi'(x^*_N - x^*_{N-1}) + \psi'(x^*_N)
\end{aligned}
\end{equation}
So, $x^*_i - x^*_{i+1} = \zeta$ for $i < N$, and $x^*_N = \zeta$, therefore,
\begin{equation}
x^* = \zeta\sum_{i=1}^N (N-i+1)e_i
\end{equation}
The minimizer has norm
\begin{equation}
\nrm*{x^*}^2
= \zeta^2\sum_{i=1}^N (N-i+1)^2 
= \frac{\zeta^2}{6}\prn*{2N^3 + 3N^2 + N}
\end{equation}
We therefore choose $\zeta^2 = \frac{B^2}{N^3}$ so that $\nrm{x^*} \leq B$. In this case,
\begin{align}
\min_{x:\nrm{x}\leq B} F(x) 
&= F(x^*) \\
&= \psi'(-\zeta) x^*_1 + \psi(x^*_N) + \sum_{i=1}^{N-1}\psi\prn*{x^*_{i+1} - x^*_i} \\
&= N\zeta\psi'(-\zeta)  + N\psi\prn*{\zeta} \\
&= -N\zeta\frac{\sqrt{H}}{2\beta}\arctan\prn*{\frac{\sqrt{H}\beta \zeta}{2}} \\
&\qquad +  N\brk*{\frac{\sqrt{H}\zeta}{2\beta}\arctan\prn*{\frac{\sqrt{H}\beta \zeta}{2}} - \frac{1}{2\beta^2}\log\prn*{1+\frac{H\beta^2\zeta^2}{4}}}\nonumber \\
&= -\frac{N}{2\beta^2}\log\prn*{1+\frac{HB^2\beta^2}{4N^3}}
\end{align}
Now, consider some $x$ such that $\prg(x) = n \leq \frac{N}{2}$, and observe that
\begin{align}
F(x)
&= \psi'(-\zeta) x_1 + \psi(x_N) + \sum_{i=1}^{N-1}\psi\prn*{x_{i+1} - x_i} \\
&= \psi'(-\zeta) x_1 + \psi(x_n) + \sum_{i=1}^{n-1}\psi\prn*{x_{i+1} - x_i}
\end{align}
Therefore, by the same argument as above, 
\begin{equation}
F(x) \geq -\frac{n}{2\beta^2}\log\prn*{1+\frac{HB^2\beta^2}{4N^3}}
\end{equation}
and we conclude that
\begin{equation}
F(x) - F^* \geq \frac{N}{4\beta^2}\log\prn*{1+\frac{HB^2\beta^2}{4N^3}}
\end{equation}
From here, we consider two cases, if $\beta^2 > \frac{4N^3}{HB^2}$, then
\begin{equation}
F(x) - F^* > \frac{N}{4\beta^2}\log\prn*{2} > \frac{N}{6\beta^2}
\end{equation}
Otherwise, if $\beta^2 \leq \frac{4N^3}{HB^2}$ then we use that for $x \leq 1$, $\log(1+x) \geq \frac{x}{2}$ and conclude
\begin{equation}
F(x) - F^* > \frac{N}{4\beta^2} > \frac{N}{6\beta^2}\cdot \frac{HB^2\beta^2}{8N^3} = \frac{HB^2}{48N^2}
\end{equation}
This completes the proof.
\end{proof}

We are now ready to prove a lower bound in terms of $\beta$, which Theorem \ref{thm:bounded-third} instantiates for different constraints on the objective:
\begin{lemma}\label{lem:beta-lower-bound}
For any $H,B,\sigma,K,R,\beta > 0$ and any $M \geq 2$, there exists a convex, $H$-smooth objective $F$ with $\nrm{x^*}\leq B$ and a stochastic gradient oracle $g$ with $\E\nrm{g(x) - \nabla F(x)}^2 \leq \sigma^2$ for all $x$ such that with probability at least $\frac{1}{2}$, all of the oracle queries, $\crl{x^m_{k,r}}$, made by any distributed zero-respecting intermittent communication algorithm have suboptimality
\[
\min_{m,k,r}F(x^m_{k,r}) - F^* 
\geq  c\cdot\brk*{\frac{HB^2}{K^2R^2} + \min\crl*{\frac{\sigma B}{\sqrt{MKR}},HB^2} + \min\crl*{\frac{HB^2}{R^2 \log^2 M},\, \frac{\beta^4 \sigma^4 B^2}{H K^2R^2},\, \frac{\sigma B}{\sqrt{KR}}}}
\]
Furthermore, the objective $F$ is simultaneously $\beta$-self-concordant, $\frac{\sqrt{H}\beta}{2}$-quasi-self-concordant, and has $\sup_{x,u} \abs{\nabla^3 F(x)[u,u,u]} \leq \frac{4H^{3/2}\beta}{3}\nrm{u}^3$.
\end{lemma}
\begin{proof}
By Lemma \ref{lem:F-properties5}, $F$ as defined in \eqref{eq:def-F5} is $H$-smooth, convex, $\beta$-self-concordant, $\frac{\sqrt{H}\beta}{2}$-quasi-self-concordant, and has $\sup_{x,u} \abs{\nabla^3 F(x)[u,u,u]} \leq \frac{4H^{3/2}\beta}{3}\nrm{u}^3$. Furthermore, by Lemma \ref{lem:stochastic-gradient-oracle}, $g$ as defined in \eqref{eq:def-g} is unbiased and has variance bounded by $\sigma^2$ for the choice of $p$ used in Lemma \ref{lem:progress}. Finally, when we choose $\zeta^2 = \frac{B^2}{N^3}$, the minimizer of $F$ has norm $\nrm{x^*} \leq B$ by Lemma \ref{lem:suboptimality}. Therefore, the objective $F$ and stochastic gradient oracle $g$ are suitable for the lower bound.

By Lemma \ref{lem:progress}, with probability at least $\frac{1}{2}$, all of the iterates of any distributed zero-respecting intermittent communication algorithm will have progress at most
\begin{equation}
\max_{m,k,r}\prg(x^m_{k,r}) \leq \min\crl*{RK,\,\max\crl*{48 R \log M,\, \frac{4\pi^2 HKR}{\pi^2 H + 8\sigma^2 \beta^2}}}
\end{equation}
For the rest of the proof, we condition on this event and set
\begin{equation}
N = \min\crl*{2KR,\,\max\crl*{96 R \log M,\, \frac{8\pi^2 HKR}{\pi^2 H + 8\sigma^2 \beta^2}}}
\end{equation}
so that $\max_{m,k,r}\prg(x^m_{k,r}) \leq \frac{N}{2}$. By Lemma \ref{lem:suboptimality}, this means that 
\begin{equation}
\min_{m,k,r}F(x^m_{k,r}) - F^* \geq \begin{cases}
\frac{N}{6\beta^2} & \beta^2 > \frac{4N^3}{HB^2} \\
\frac{HB^2}{48N^2} & \beta^2 \leq \frac{4N^3}{HB^2}
\end{cases}
\end{equation}
Thus, if $\beta^2 \leq \frac{4N^3}{HB^2}$, then
\begin{equation}
\min_{m,k,r}F(x^m_{k,r}) - F^* 
\geq \frac{HB^2}{48N^2} 
= \frac{HB^2}{48\min\crl*{4K^2R^2,\,\max\crl*{9216 R^2 \log^2 M,\, \frac{64\pi^4 H^2K^2R^2}{(\pi^2 H + 8\sigma^2 \beta^2)^2}}}} \label{eq:lower-bound-thm-eq1}
% \\
% &\geq \frac{HB^2}{96}\brk*{\frac{1}{4K^2R^2} + \min\crl*{\frac{1}{9216 R^2 \log^2 M},\, \frac{\sigma^4 \beta^4}{\pi^4 H^2K^2R^2}}}
\end{equation}
Therefore, there is a universal constant $c$ such that
\begin{equation}\label{eq:lower-bound-thm-eq3}
\min_{m,k,r}F(x^m_{k,r}) - F^* 
\geq c\cdot\prn*{\frac{HB^2}{K^2R^2} + \min\crl*{\frac{HB^2}{R^2 \log^2 M},\, \frac{\sigma^4 \beta^4 B^2}{HK^2R^2}}}
\end{equation}
% We note that we could instantiate the lower bound with any $H' \leq H$ and the objective would still be $H$-smooth. Therefore, in the event that $\frac{\sigma^4 \beta^4 B^2}{HK^2R^2} \leq \frac{HB^2}{R^2 \log^2 M}$, we use $H' = \frac{\sigma^2\beta^2\log M}{K}$ and conclude
% \begin{equation}
% \min_{m,k,r}F(x^m_{k,r}) - F^* 
% \geq c\cdot\prn*{\frac{HB^2}{K^2R^2} + \min\crl*{\frac{HB^2}{R^2 \log^2 M},\, \frac{\sigma^2 \beta^2 B^2}{KR^2\log M}}}\label{eq:lower-bound-thm-eq3}
% \end{equation}

On the other hand, if $\beta^2 > \frac{4N^3}{HB^2}$, since $N = N(\beta)$ is a non-increasing function of $\beta$, we can always instantiate the lower bound in terms of $\beta' < \beta$ such that $\beta'^2 = \frac{4N(\beta')^3}{HB^2}$. With this choice, in light of \eqref{eq:lower-bound-thm-eq1}, there is a universal constant $c$ such that
\begin{equation}
\min_{m,k,r}F(x^m_{k,r}) - F^* 
\geq  c\cdot\brk*{\frac{HB^2}{K^2R^2} + \min\crl*{\frac{HB^2}{R^2 \log^2 M},\, \frac{(\pi^2 H + 8\sigma^2 {\beta'}^2)^2B^2}{HK^2R^2}}}\label{eq:lower-bound-thm-eq2}
\end{equation}
Furthermore, with our choice of $\beta'$,
\begin{align}
{\beta'}^2 
&= \frac{4N(\beta')^3}{HB^2} \\
&= \frac{4}{HB^2}\min\crl*{8K^3R^3,\,\max\crl*{96^3 R^3 \log^3 M,\, \frac{8^3\pi^6 H^3K^3R^3}{(\pi^2 H + 8\sigma^2 {\beta'}^2)^3}}} \\
&\geq \frac{4K^3R^3}{HB^2}\min\crl*{8,\,\frac{8^3\pi^6 H^3}{(\pi^2 H + 8\sigma^2 {\beta'}^2)^3}} \\
&\geq \frac{32K^3R^3}{HB^2}\frac{\pi^6 H^3}{(\pi^2 H + 8\sigma^2 {\beta'}^2)^3}
\end{align}
Therefore,
\begin{align}
\frac{1}{8\sigma^2}(\pi^2 H + 8\sigma^2 {\beta'}^2)^4 &\geq {\beta'}^2(\pi^2 H + 8\sigma^2 {\beta'}^2)^3 \geq \frac{32\pi^6H^2K^3R^3}{B^2} \\
\implies (\pi^2 H + 8\sigma^2 {\beta'}^2)^2 &\geq \frac{16\pi^3 H\sigma K^{3/2}R^{3/2}}{B}
\end{align}
In light of \eqref{eq:lower-bound-thm-eq2}, we conclude that
\begin{equation}
\min_{m,k,r}F(x^m_{k,r}) - F^* 
\geq  c\cdot\brk*{\frac{HB^2}{K^2R^2} + \min\crl*{\frac{HB^2}{R^2 \log^2 M},\, \frac{\sigma B}{\sqrt{KR}}}}\label{eq:lower-bound-thm-eq4}
\end{equation}
Combining Lemma \ref{lem:statistical-term-lower-bound} with \eqref{eq:lower-bound-thm-eq3} and \eqref{eq:lower-bound-thm-eq4} completes the proof. We note that the lower bound in Lemma \ref{lem:statistical-term-lower-bound} is achieved by a quadratic hard instance, which is $0$-self-concordant, $0$-quasi-self-concordant, and has $0$-Lipschitz Hessian.
\end{proof}

\bdedthird*
\begin{proof}
By Lemma \ref{lem:beta-lower-bound}, the objective $F$ and stochastic gradient oracle $g$ satisfy the necessary conditions with $F$ having $\nrm{\nabla^3 F(x)} \leq \frac{4H^{3/2}\beta}{3}$. In terms of $\beta$ and the other problem parameters, the lower bound is then
\begin{multline}%\label{eq:thm-3-generic}
\min_{m,k,r}F(x^m_{k,r}) - F^* 
\geq  c\cdot\brk*{\frac{HB^2}{K^2R^2} + \min\crl*{\frac{\sigma B}{\sqrt{MKR}},HB^2} + \min\crl*{\frac{HB^2}{R^2 \log^2 M},\, \frac{\beta^4 \sigma^4 B^2}{H K^2R^2},\, \frac{\sigma B}{\sqrt{KR}}}}
\end{multline}
with probability $\frac{1}{2}$.
Below, we will use the fact that we can instantiate the lower bound in terms of any $H' \leq H$ without invalidating the result since an $H'$-smooth objective is also $H$-smooth. This allows us to maximize the lower bound over $H' \leq H$ to achieve a tighter result. Choosing $\beta = \frac{3Q}{4H^{3/2}}$ ensures that $\nrm{\nabla^3 F(x)} \leq Q$, i.e.~$\nabla^2 F$ is $Q$-Lipschitz, so
\begin{multline}
\min_{m,k,r}F(x^m_{k,r}) - F^* 
\geq  c\cdot\brk*{\frac{HB^2}{K^2R^2} + \min\crl*{\frac{\sigma B}{\sqrt{MKR}},HB^2} + \min\crl*{\frac{HB^2}{R^2 \log^2 M},\, \frac{Q^4 \sigma^4 B^2}{H^7 K^2R^2},\, \frac{\sigma B}{\sqrt{KR}}}}
\end{multline}
In the event that the $\frac{Q^4 \sigma^4 B^2}{H^7 K^2R^2}$ term is the minimizer, we can take $H' = \frac{\sqrt{Q\sigma}\log^{1/4} M}{K^{1/4}} \leq H$ to conclude
\begin{multline}
\min_{m,k,r}F(x^m_{k,r}) - F^* \\
\geq  c\cdot\brk*{\frac{HB^2}{K^2R^2} + \min\crl*{\frac{\sigma B}{\sqrt{MKR}},HB^2} + \min\crl*{\frac{HB^2}{R^2 \log^2 M},\, \frac{\sqrt{Q \sigma} B^2}{K^{1/4}R^2\log^{7/4} M},\, \frac{\sigma B}{\sqrt{KR}}}}
\end{multline}
This completes the proof.
\end{proof}

\begin{theorem}
For any $H,B,\sigma,Q,K,R > 0$ and any $M \geq 2$, there exists a convex, $H$-smooth objective $F$ with $\nrm{x^*}\leq B$ and a stochastic gradient oracle $g$ with $\E\nrm{g(x) - \nabla F(x)}^2 \leq \sigma^2$ for all $x$, such with probability at least $\frac{1}{2}$ all of the oracle queries $\crl{x^m_{k,r}}$ made by any distributed-zero-respecting intermittent communication algorithm will have suboptimality lower bounded as follows: \\
When the objective, $F$, is required to be $Q$-self-concordant
\[
\min_{m,k,r}F(x^m_{k,r}) - F^*
\geq  c\cdot\brk*{\frac{HB^2}{K^2R^2} + \min\crl*{\frac{\sigma B}{\sqrt{MKR}},HB^2} + \min\crl*{\frac{HB^2}{R^2 \log^2 M},\, \frac{Q^2 \sigma^2 B^2}{KR^2\log M},\, \frac{\sigma B}{\sqrt{KR}}}}
\]
When the objective, $F$, is required to be $Q$-quasi-self-concordant
\[
\min_{m,k,r}F(x^m_{k,r}) - F^*
\geq  c\cdot\brk*{\frac{HB^2}{K^2R^2} + \min\crl*{\frac{\sigma B}{\sqrt{MKR}},HB^2} + \min\crl*{\frac{HB^2}{R^2 \log^2 M},\, \frac{Q \sigma B^2}{\sqrt{K}R^2\log^{3/2} M},\, \frac{\sigma B}{\sqrt{KR}}}}
\]
\end{theorem}
\begin{proof}
By Lemma \ref{lem:beta-lower-bound}, the objective $F$ and stochastic gradient oracle $g$ satisfy the necessary conditions with $F$ being $\beta$-self concordant and $\frac{\sqrt{H}\beta}{2}$-quasi-self-concordant. In terms of $\beta$ and the other problem parameters, the lower bound is then
\begin{multline}%\label{eq:thm-3-generic}
\min_{m,k,r}F(x^m_{k,r}) - F^* 
\geq  c\cdot\brk*{\frac{HB^2}{K^2R^2} + \min\crl*{\frac{\sigma B}{\sqrt{MKR}},HB^2} + \min\crl*{\frac{HB^2}{R^2 \log^2 M},\, \frac{\beta^4 \sigma^4 B^2}{H K^2R^2},\, \frac{\sigma B}{\sqrt{KR}}}}
\end{multline}
with probability $\frac{1}{2}$.
Below, we will use the fact that we can instantiate the lower bound in terms of any $H' \leq H$ without invalidating the result since an $H'$-smooth objective is also $H$-smooth. This allows us to maximize the lower bound over $H' \leq H$ to achieve a tighter result. We proceed by considering the two cases separately.

\paragraph{Self-Concordance:} Choosing $\beta = Q$ ensures that $F$ is $Q$-self-concordant, so 
\begin{multline}
\min_{m,k,r}F(x^m_{k,r}) - F^* 
\geq  c\cdot\brk*{\frac{HB^2}{K^2R^2} + \min\crl*{\frac{\sigma B}{\sqrt{MKR}},HB^2} + \min\crl*{\frac{HB^2}{R^2 \log^2 M},\, \frac{Q^4 \sigma^4 B^2}{H K^2R^2},\, \frac{\sigma B}{\sqrt{KR}}}}
\end{multline}
In the event that the $\frac{Q^4 \sigma^4 B^2}{H K^2R^2}$ term is the minimizer, we can take $H' = \frac{Q^2\sigma^2 \log M}{K} \leq H$ to conclude that 
\begin{multline}
\min_{m,k,r}F(x^m_{k,r}) - F^* 
\geq  c\cdot\brk*{\frac{HB^2}{K^2R^2} + \min\crl*{\frac{\sigma B}{\sqrt{MKR}},HB^2} + \min\crl*{\frac{HB^2}{R^2 \log^2 M},\, \frac{Q^2 \sigma^2 B^2}{KR^2\log M},\, \frac{\sigma B}{\sqrt{KR}}}}
\end{multline}

\paragraph{Quasi-self-Concordance:} Choosing $\beta = \frac{2Q}{\sqrt{H}}$ ensures that $F$ is $Q$-quasi-self-concordant, so
\begin{multline}
\min_{m,k,r}F(x^m_{k,r}) - F^* 
\geq  c\cdot\brk*{\frac{HB^2}{K^2R^2} + \min\crl*{\frac{\sigma B}{\sqrt{MKR}},HB^2} + \min\crl*{\frac{HB^2}{R^2 \log^2 M},\, \frac{Q^4 \sigma^4 B^2}{H^3 K^2R^2},\, \frac{\sigma B}{\sqrt{KR}}}}
\end{multline}
In the event that the $\frac{Q^4 \sigma^4 B^2}{H^3 K^2R^2}$ term is the minimizer, we can take $H' = \frac{Q \sigma \log^{1/2} M}{\sqrt{K}} \leq H$ to conclude that
\begin{multline}
\min_{m,k,r}F(x^m_{k,r}) - F^* 
\geq  c\cdot\brk*{\frac{HB^2}{K^2R^2} + \min\crl*{\frac{\sigma B}{\sqrt{MKR}},HB^2} + \min\crl*{\frac{HB^2}{R^2 \log^2 M},\, \frac{Q \sigma B^2}{\sqrt{K}R^2\log^{3/2} M},\, \frac{\sigma B}{\sqrt{KR}}}}
\end{multline}
This completes the proof.
\end{proof}

\end{document}